\documentclass[11pt, letterpaper]{article}
\usepackage[margin=1in]{geometry}
\usepackage[pdftex]{graphicx}
\usepackage[utf8]{inputenc} 
\usepackage[T1]{fontenc}    
\hypersetup{%
colorlinks=true,
urlcolor=blue,
citecolor=blue,
linkcolor=red
}
\usepackage{breakurl}
\usepackage{booktabs}       
\usepackage{amsfonts}       
\usepackage{nicefrac}       
\usepackage{microtype}      
\usepackage{amsmath,amssymb}
\usepackage{bm}
\usepackage[style=base]{caption}
\usepackage[numbers,sort&compress]{natbib}

\usepackage{amsthm}         
\newtheorem{thm}{Theorem}
\newtheorem{cor}{Corollary}

\theoremstyle{remark}

\usepackage{epstopdf}
\usepackage{color}
\usepackage{subfig}
\usepackage{mleftright}
\usepackage{multirow}
\usepackage{array}

\usepackage{algorithmicx}
\usepackage{algorithm}
\usepackage[noend]{algpseudocode}
\algnewcommand\algorithmicinput{\textbf{Input:}}
\algnewcommand\Input{\item[\algorithmicinput]}

\algnewcommand\algorithmicoutput{\textbf{Output:}}
\algnewcommand\Output{\item[\algorithmicoutput]}

\usepackage{xr}
\usepackage{makecell}

\newcommand{\Pxy}{P_{X,Y}}

\newcommand{\Pygx}{P_{Y|X}}
\newcommand{\Px}{P_{X}}
\newcommand{\Py}{P_{Y}}

\newcommand{\EE}[1]{\mathbb{E}\left[#1\right]}

\newcommand{\calX}{\mathcal{X}}

\newcommand{\calN}{\mathcal{N}}
\newcommand{\calY}{\mathcal{Y}}

\newcommand{\calS}{\mathcal{S}}

\newcommand{\calF}{\mathcal{F}}
\newcommand{\calG}{\mathcal{G}}

\newcommand{\calL}{\mathcal{L}}

\newcommand{\bff}{\mathbf{f}}
\newcommand{\bg}{\mathbf{g}}
\newcommand{\bSigma}{\bm{\Sigma}}

\newcommand{\bA}{\mathbf{A}}
\newcommand{\bB}{\mathbf{B}}
\newcommand{\bC}{\mathbf{C}}

\newcommand{\bG}{\mathbf{G}}

\newcommand{\bF}{\mathbf{F}}
\newcommand{\bI}{\mathbf{I}}

\newcommand{\bU}{\mathbf{U}}
\newcommand{\bV}{\mathbf{V}}

\newcommand{\bL}{\mathbf{L}}

\newcommand{\bY}{\mathbf{Y}}
\newcommand{\bZ}{\mathbf{Z}}
\newcommand{\bx}{\mathbf{x}}
\newcommand{\by}{\mathbf{y}}

\newcommand{\bX}{\mathbf{X}}

\renewcommand{\tilde}{\widetilde}

\newcommand{\Reals}{\mathbb{R}}
\newcommand{\mtf}{\mathbf{\tilde{f}}}
\newcommand{\mtg}{\mathbf{\tilde{g}}}

\newcommand\blfootnote[1]{%
  \begingroup
  \renewcommand\thefootnote{}\footnote{#1}%
  \addtocounter{footnote}{-1}%
  \endgroup
}

\begin{document}
\title{Generalizing Correspondence Analysis\\for Applications in Machine Learning}
\date{}
\author{
    Hsiang~Hsu\thanks{H. Hsu and F. P. Calmon are with the John A. Paulson School of Engineering and Applied Sciences, Harvard University, Cambridge, MA, 02138. E-mails: \texttt{hsianghsu@g.harvard.edu, flavio@seas.harvard.edu}.},
    Salman~Salamatian\thanks{S. Salamatian is with the Research Laboratory of Electronics, Massachusetts Institute of Technology, Cambridge, MA, 02139. E-mail: \texttt{salmansa@mit.edu}.},
    and~Flavio~du~Pin~Calmon${}^*$\blfootnote{This work was presented in part at the International Conference on Artificial Intelligence and Statistics (AISTATS) in 2019 \cite{hsu2019correspondence}. In this submitted manuscript, we expand the principal inertia component-based correspondence analysis to applications in visualizations of classification boundary and its training process, multi-view learning, and multi-modal learning, both in theory and with experiments.}
}




\maketitle

\begin{abstract}
Correspondence analysis (CA) is a multivariate statistical tool used to visualize and interpret data dependencies by finding maximally correlated embeddings of pairs of random variables. 
CA has found applications in fields ranging from epidemiology to social sciences; however, current methods do not scale to large, high-dimensional datasets.
In this paper, we provide a novel interpretation of CA in terms of an information-theoretic quantity called the principal inertia components. We show that estimating the principal inertia components, which consists in solving a functional optimization problem over the space of finite variance functions of two random variable, is equivalent to performing CA. We then leverage this insight to design novel algorithms to perform CA at an unprecedented scale.
Particularly, we demonstrate how the principal inertia components can be reliably approximated from data using deep neural networks. Finally, we show how these maximally correlated embeddings of pairs of random variables in CA further play a central role in several learning problems including visualization of classification boundary and training process, and underlying recent multi-view and multi-modal learning methods. 
\end{abstract}


\section{Introduction}\label{sec:introduction}
Correspondence Analysis (CA) is an exploratory multivariate statistical technique with a decades-long history in applied statistics \cite{benzecri1973correspondence, greenacre1984theory, lebart2013correspondence, greenacre2017correspondence}.
CA aims to produce interpretable, low-dimensional visualizations (often two-dimensional) that capture complex relationships in data with entangled and intricate dependencies, leading to its successful deployment in fields ranging from bioinformatics and epidemiology \cite{tekaia2016genome, busold2005integration, sourial2010correspondence} to social, environmental and market sciences \cite{carrington2005models, ter2004co, ormoli2015diversity, ferrari2016whole, hoare2019brand}. 

Despite being a versatile statistical technique, CA has been underutilized on the large datasets with high-dimensional entries found in the current machine learning landscape.
This is due, at least in part, to the fact that existing methods for performing CA are not scalable. Traditionally, CA uses as its main computational ingredient a singular value decomposition (SVD) of the normalized contingency table\footnote{A contingency table is the empirical estimate of the joint distribution or, in other words, the table with the relative frequency of each observation of two random variables. See  \eqref{eq:contengency_table} for a precise definition.} of i.i.d. realizations of two random variables.  
This contingency table-based approach (see Section~\ref{sec:ca} for details) for performing CA has three fundamental limitations:
\begin{enumerate}
    \item It is restricted to data drawn from \emph{discrete} distributions with finite support, since contingency tables for continuous variables will be highly dependent on a chosen quantization.
    \item Even when the underlying distribution of the data is discrete, reliably estimating the contingency table may be infeasible due to \emph{limited number of samples}. This inevitably hinges CA on the more (statistically) challenging problem of estimating a joint distribution.
    \item It is not feasible to build contingency tables for \emph{high-dimensional} data. For example, if a random variable $X \in \{0, 1\}^a$ and all outcomes have non-zero probability, then the contingency table has $2^a$ rows.
\end{enumerate}

To address these limitations of contingency table-based CA and in order to scale-up this technique to large datasets, we revisit the core mathematical problem behind CA. Given two variables $X$ and $Y$ and the $d$-dimensional embeddings 
\begin{eqnarray}
\label{eq:non-linear}
\begin{aligned}
    \mathbf{f}(X) &= \left[f_1(X),\dots,f_d(X) \right]^\top \in \Reals^d,\\
    \mathbf{g}(Y) &= \left[g_1(Y),\dots,g_d(Y) \right]^\top \in \Reals^d
\end{aligned}
\end{eqnarray}
of $X$ and $Y$, respectively, we argue that CA seeks to find $\bff(X)$ and $\bg(Y)$ which are \emph{maximally correlated}, i.e., $\max \EE{\mathbf{f}(X)^\top\mathbf{g}(Y)}$, such that $\EE{\bff(X)}$ and $\EE{\mathbf{g}(Y)}$ are orthonormal\footnote{$\EE{\bff(X)\bff(X)^\top}=\mathbf{I}_d$ and $\EE{\mathbf{g}(Y)\mathbf{g}(Y)^\top}=\mathbf{I}_d$, where $\mathbf{I}_d$ is the $d$-dimensional identity matrix.}.  
Maximally correlated embeddings of two variables $X$ and $Y$ are also sought by statistical methods such as Canonical Correlation Analysis (CCA) \cite{hotelling1936relations} and its kernel variants \cite{bach2002kernel, hardoon2004canonical}. In contrast, CA is not limited to linear or kernel embeddings, but produces general non-linear embeddings.
These embeddings also play a central role in a broad range of learning tasks, from exploratory visualization/interpretation \cite{greenacre1984theory, greenacre1987geometric, greenacre2017correspondence}, data clustering \cite{zhang2018generalized, chaudhuri2009multi}, to multi-view and multi-modal learning \cite{arora2013multi, hu2018sharable, li2018survey, wang2015deep, benton2017deep}.

In this paper, we generalize CA beyond the contingency table-based method by considering the embeddings $\{f_i, g_i \}_{i=1}^d$ in the Hilbert space of functions with finite-variance that are pairwise maximally correlated, i.e., 
\begin{eqnarray}
\label{eq:non-linear2}
\max\limits_{f_i, g_i} \sqrt{\EE{f_i(X) g_i(Y)}} \triangleq \lambda_i,\;\forall i = 1, \cdots, d.
\end{eqnarray} 
We bring to bear an information-theoretic tool called the \emph{Principal Inertia Components} (PICs) of a joint distribution $\Pxy$ \cite{greenacre1984theory, du2017principal, buja1990remarks, renyi1959measures, witsenhausen1975sequences, makur2016polynomial}, that connects analysis of functions in Hilbert spaces with the Mean-Square-Error (MSE) and the $\chi^2$ statistic (Section~\ref{sec:pics_principal functions}).
Here, the orthonormal set of low-dimensional embeddings $\{f_i, g_i \}_{i=1}^d$ are known as the principal functions \cite{du2017principal} and the correlations $\lambda_i$ are known as the PICs.
The PICs and the principal functions possess two important properties: (i) The embeddings $\{f_i, g_i \}_{i=1}^d$ are the (usually non-linear) maximally correlated functions of $X$ and $Y$ in a descending order of correlation, i.e., $1 \geq \lambda_1 \geq \cdots \geq \lambda_d \geq 0$, and (ii) $\{f_i, g_i \}_{i=1}^d$ and $\{\lambda_i \}_{i=1}^d$ reconstruct the joint distribution between $X$ and $Y$ (cf. Corollary~\ref{corr:pic} and Theorem~\ref{thm:reconstitution_formula} in Section~\ref{sec:lancaster-decomposition}).

We further demonstrate, both in theory and in practice, that the low-dimensional embeddings produced by CA with discrete variables are \emph{special cases} of the principal functions found in the theory of PICs (Section~\ref{sec:generalizing_CA}).
We leverage this connection to recast CA in terms of a functional optimization that seeks to produce maximally correlated embeddings, and also design a deep neural network to estimate the PICs and principal functions. 
This network, named the Principal Inertia Component Estimator (PICE), can significantly scale-up CA methods to data relevant to current machine learning problems (e.g., images and texts).
Interestingly, the PICE can accurately estimate the principal functions for both discrete and continuous (potentially high-dimensional) random variables in practice, and our experiments show that it correctly approximates theoretical PICs for synthetic data (Section~\ref{sec:pice}).

With the generalized CA at our disposal, we study several use cases of interest in statistics, data science, and machine learning.
First, in Section~\ref{sec:gca}, we tackle visualization and interpretation of large datasets using CA.
Second, in Section~\ref{sec:reconstruct_pics}, we introduce a method to perform CA on black-box classification models using PICs and demonstrate (i) how to visualize the classification boundary of unknown classifiers and identify ambiguous samples, (ii) how to visualize the training behaviors (e.g, which class is first learnt/hardest to learn) of a classification model.
Finally, in Section~\ref{sec:mv_mm}, we use the generalized CA to quantify the dimensionality of the common latent space shared by different views in multi-view learning \cite{zhang2018generalized, hu2018sharable, wang2015deep}, and to visualize and interpret the correlations between data from different domains (e.g., images and captions) in multi-modal learning \cite{li2018survey, ngiam2011multimodal, srivastava2012multimodal}.
All these use cases are illustrated trhough extensive experiments using real-world datasets including images (e.g., MNIST \cite{lecun1998gradient}, noisy MNIST \cite{wang2015deep}, and CIFAR-10 \cite{krizhevsky2009learning}), the Flickr-$30$k image caption dataset \cite{plummer2015flickr30k}, recipes \cite{kaggle_what_cooking}, and others \cite{asuncion2007uci, aaker1997dimensions}.

In the remainder of this section, we survey related work and define notations.
Mathematical proofs and experimental details are provided in the Appendices, and Section~\ref{sec:conclusion} concludes this paper.
Source code for reproducing our experimental results are available\footnote{In this pre-print version of the manuscript, we add hyperlinks to relevant content (e.g., code and videos) directly to the text for ease of access. These will be converted into references in an eventual final version.} at \url{https://github.com/HsiangHsu/CorrelatedEmbeddings}.

\subsection{Related Work}
Several statistical methods exist to produce correlated low-dimensional embeddings of two variables $X$ and $Y$. 
A widely used approach is CCA \cite{hotelling1936relations}, which seeks to find the best linear relationships between the random variables. 
Kernel CCA (KCCA) \cite{bach2002kernel, hardoon2004canonical} extends CCA by first projecting the variables onto a Reproducing Kernel Hilbert Space (RKHS).
The (bivariate) Alternating Conditional Expectations (ACE) algorithm only considers non-linear embeddings $f_1$ and $g_1$ that are maximally correlated whilst having zero mean and unit variance \cite{breiman1985estimating, buja1990remarks}. 
The embeddings are found by iteratively computing $\EE{g(Y)|X}$ and $\EE{f(X)|Y}$, and converge to $f_1$ and $g_1$.
More closely related to the task of determining the PICs is the Deep CCA method \cite{andrew2013deep}, where non-linear mappings of multi-view data are produced using neural networks. 
In fact, the objective function of DCCA in \cite[Eq. 1]{wang2015deep} is similar to the definition of the PICs.
However, the non-linear projections found by DCCA are not exactly the principal functions, and an exact connection between DCCA and CA or PICs is not stated in \cite{wang2015deep}.

Despite being widely used, these methods have several limitations.
In CCA/KCCA, the correlations are restricted to predefined linear embeddings or RKHS.
In the ACE algorithm, the correlations beyond $f_1$ and $g_1$ are neglected, and for large, high-dimensional datasets, iteratively computing conditional expectations is intractable due to high sample complexity \cite{huang2019sample}.
In DCCA, the embeddings $\bff$ and $\bg$ are hard to visualize and interpret.
In comparison, the PICs and principal functions not only generalize the aforementioned methods for producing maximally correlated embeddings, such as (K)CCA and the (bivariate) ACE algorithm to much larger datasets via deep learning models, the Hilbert space perspective presented here also equips DCCA with a greater theoretical significance beyond its current use in multi-view and multi-modal learning.

Under different guises, the PICs (i.e., finding the most correlated non-linear embeddings of two variables) date back to the seminal papers in the information theory and statistics literature by Hirschfeld \cite{hirschfeld1935connection}, Gebelein \cite{gebelein1941statistische}, R\'enyi \cite{renyi1959measures}, and others \cite{lancaster1958structure, hannan1961general, witsenhausen1975sequences}, long before the derivation of (K)CCA, DCCA, and the ACE algorithm.
PICs are a generalization of the (Hirschfeld-Gebelein-R\'enyi) maximal correlation; in fact, the first PIC $\lambda_1$ is identical to the maximal correlation \cite{buja1990remarks}.
This groundbreaking prior work characterized the optimal embeddings of pairs of random variables without restrictions to an RKHS or need of a specific  parametric form for the embeddings. However, until now, these powerful theoretical results have not inspired practical algorithms or use cases for estimating non-linear embeddings.
Recently, the theory behind PICs was revisited in \cite{makur2015bounds, huang2017information, du2017principal} for information-theoretic use cases, e.g., analyzing local approximations of $f$-divergences \cite{makur2015linear, csiszar2004information} and in privacy \cite{wang2017estimation}. 
Unlike prior efforts, the method that we propose here (namely the PICE) allows the PICs and principal functions to be computed \emph{in practice and at scale}, thus extending the reach of the existing theory on PICs.

PICs are the common thread underlying seemingly different statistical methods such as CA, the ACE algorithm, and DCCA \cite{andrew2013deep}. 
These three methods are theoretically equivalent in that they seek to produce a (partial) PIC characterization of a joint distribution $\Pxy$ from its samples $\left\{(x_k,y_k)\right\}_{k=1}^n$. 
The connection between the PICs and DCCA is addressed above; the connection between CA and the bivariate ACE algorithm was noted by \cite{buja1990remarks}, and the connection between the PICs and the CA will be studied in Section~\ref{sec:deep_ca}. 
Recently, the sum of the $k$-largest PICs, defined as the $k$-correlation in \cite{du2017principal}, was applied in feature extraction of multi-modal data using neural networks in \cite{wang2019efficient}.
Here, we illustrate multi-modal feature extraction for image tagging as one of the many use cases of the PICs in Section~\ref{sec:mv_mm}.

We compare the PICs with Principal Component Analysis (PCA) and its variants, e.g., the kernel PCA \cite{hoffmann2007kernel} and maximally correlated PCA \cite{feizi2017maximally}. 
Similar to CA, these methods aim at representing data in terms of orthogonal (uncorrelated) components. 
As such, they capture the structural relationship within a high dimensional random vector of features $X$. 
These methods are focused on a single variable (e.g., an image) and usually used in an unsupervised manner. 
Instead, CA finds orthogonal components of both $X$ and $Y$ \emph{jointly} (e.g., images and captions), with the resulting components being highly correlated. This is fundamentally different from performing PCA or kPCA on the joint pair $(X,Y)$, as evidenced\footnote{In a similar vein, applying PCA to two concatenated random variables $(X,Y)$ does not necessarily recover CCA.} by the derivations in Section \ref{sec:deep_ca}. 
Moreover, CA produces non-linear, highly correlated embeddings, unlike kPCA, which requires a kernel to be defined \textit{a priori}, and unlike the maximally correlated PCA \cite[Eq. 1.4]{feizi2017maximally}, where objective is to maximize the sum of eigenvalues of the covariance matrix of the nonlinear transformation of $X$ and $Y$ \emph{without} the orthonormal constraints.

\subsection{Notation}
Capital letters (e.g., $X$ and $Y$) stand for random variables, and calligraphic letters (e.g., $\calX$ and $\calF$) for sets.
For any real-valued $X$, we denote the $\ell_p$-norm of $X$ as $\|X\|_p = (\EE{X^p})^{1/p}$.
We denote the probability measure of $X\times Y$ by $\Pxy$, the conditional probability measure of $Y$ given $X$ by $\Pygx$, and the marginal probability measure of $X$ and $Y$ by $\Px$ and $\Py$ respectively.
$X \sim \Px$ represents the fact that $X$ is distributed according to $\Px$.
Lower-case letters (e.g., $x$ and $y$) represent a sample drawn from a probability distribution.
Bold capital letters (e.g., $\mathbf{X}$) and bold lower-case letters (e.g., $\mathbf{x}$) are used for matrices and vectors respectively.
The $(i,j)$ entry of a matrix $\mathbf{X}$ is given by $[\mathbf{X}]_{i,j}$. 
Finally, we use $\mathbf{I}_d$ and $\mathbf{1}_d$ for the identity matrix and the all-one vector of dimension $d$, $\mathsf{diag}(\mathbf{v})$ for the matrix with diagonal entries equal to a given $\mathbf{v} \in \Reals^d$, and $[d] = \{1, \cdots, d\}$.

\section{Correspondence Analysis beyond SVD}\label{sec:deep_ca}
In this section, we formally introduce CA, the PICs and the principal functions, and (most importantly) the generalization of CA using the PICs.
We aim at presenting a geometric characterization of the space of finite-variance embeddings in an intuitive way at the expense of some mathematical rigor, and demonstrate that the set of embeddings of the form $\mathbf{f}(X)$ and $\mathbf{g}(Y)$ in \eqref{eq:non-linear} is \emph{completely} characterized by the PICs.  
The theoretical background for the PICs and the underlying functional spaces is then used to extend CA beyond its conventional matrix factorization.
The results below hold under mild compactness assumptions of the probability distributions. A more thorough investigation of the PICs can be found in \cite{witsenhausen1975sequences, buja1990remarks, du2017principal, asoodeh2015maximal, makur2016polynomial}.

\subsection{Correspondence Analysis}\label{sec:ca}
Correspondence analysis considers two random variables $X$ and $Y$ with $|\calX|, |\calY| < \infty$, and pmf $P_{X, Y}$ (cf. \cite{greenacre1984theory} for a detailed overview). 
Given samples $\{x_k, y_k\}_{k=1}^n$ drawn independently from $P_{X, Y}$, a two-way \emph{contingency table} $\mathbf{P}_{X, Y}$ is defined as a matrix with $|\calX|$ rows and $|\calY|$ columns of normalized co-occurrence counts (i.e., the relative frequency of each outcome pair), given by
\begin{eqnarray}\label{eq:contengency_table}
[\mathbf{P}_{X,Y}]_{i,j}= \frac{\mbox{\# of observations } (x_k,y_k)=(i,j)}{n}.
\end{eqnarray}
Moreover, the marginals are defined as $\mathbf{p}_X \triangleq \mathbf{P}_{X, Y} \mathbf{1}_{|\mathcal{Y}|}$ and $\mathbf{p}_Y \triangleq \mathbf{P}_{X, Y}^T \mathbf{1}_{|\mathcal{X}|}$. 
Consider a matrix 
\begin{eqnarray}
\label{eq:Q}
\mathbf{Q}&\triangleq& \mathbf{D}_{X}^{-1/2}(\mathbf{P}_{X,Y}-\mathbf{p}_X\mathbf{p}_Y^T)\mathbf{D}_{Y}^{-1/2} = \bU \bSigma \bV^\intercal, \label{eq:svd_Q}
\end{eqnarray}
where $\mathbf{D}_{X} \triangleq \mathsf{diag}(\mathbf{p}_X)$ and $\mathbf{D}_{Y} \triangleq \mathsf{diag}(\mathbf{p}_Y)$, and $\bU \bSigma \bV^\intercal$ is the SVD of $\mathbf{Q}$. 
Let $d = \min\{ |\calX|, |\calY| \}-1$, and $\{\sigma_i\}_{i=1}^d$ be the corresponding singular values, then we have the following definitions used in CA \cite{greenacre1984theory}:
\begin{itemize}
    \item The orthogonal factors of $X$ are $\mathbf{L} \triangleq \mathbf{D}_{X}^{-1/2} \bU$.
    \item The orthogonal factors of $Y$ are $\mathbf{R} \triangleq \mathbf{D}_{Y}^{-1/2} \bV$.
    \item The factor scores are $\lambda_i = \sigma_i^2, 1 \leq i \leq d$.
    \item The factor score ratios are $\frac{\lambda_i}{\sum_{i=1}\lambda_i}, 1 \leq i \leq d$.
\end{itemize}
CA makes use of the orthogonal factors $\mathbf{L}$ and $\mathbf{R}$  to visualize the correspondence (i.e., dependencies) between $X$ and $Y$. In particular, the first two columns of $\mathbf{L}$ and $\mathbf{R}$  can be plotted on a two-dimensional plane, with $([\mathbf{L}]_{i, 1}, [\mathbf{L}]_{i, 2})$ displayed as a point, and similarly for $\mathbf{R}$.
This visualization is known as the \emph{factoring plane}, and is analogous to visualization techniques based on CCA or PCA, where data is projected onto a two-dimensional plane. 
The remaining planes can be produced by plotting the other columns of $\mathbf{L}$ and $\mathbf{R}$.
The factor score ratio quantifies the variance (``correspondence'') captured by each orthogonal factor, and is often shown along the axes in factoring planes.
CA decomposes the $\chi^2$ statistic associated with a two-way contingency table, as we shall soon see in Section~\ref{sec:pics_principal functions}. Later on (Theorem \ref{thm:CA_PIC}), we  demonstrate that the orthogonal factors are equivalent to the maximally-correlated non-linear embeddings given by the PICs in \eqref{eq:non-linear} and \eqref{eq:non-linear2}. 
Finally, CA is also related to two-way spectral clustering of an undirected bipartite graph where nodes correspond to outcomes of discrete random variable $X$ and $Y$ and  $[\mathbf{P}_{X, Y}(x,y)$ is the weight of the edge connecting node $x$ with node $y$ \cite{chung1997spectral, li2006relationships}.

\begin{figure}[t!]
\centering
\includegraphics[width=.6\textwidth]{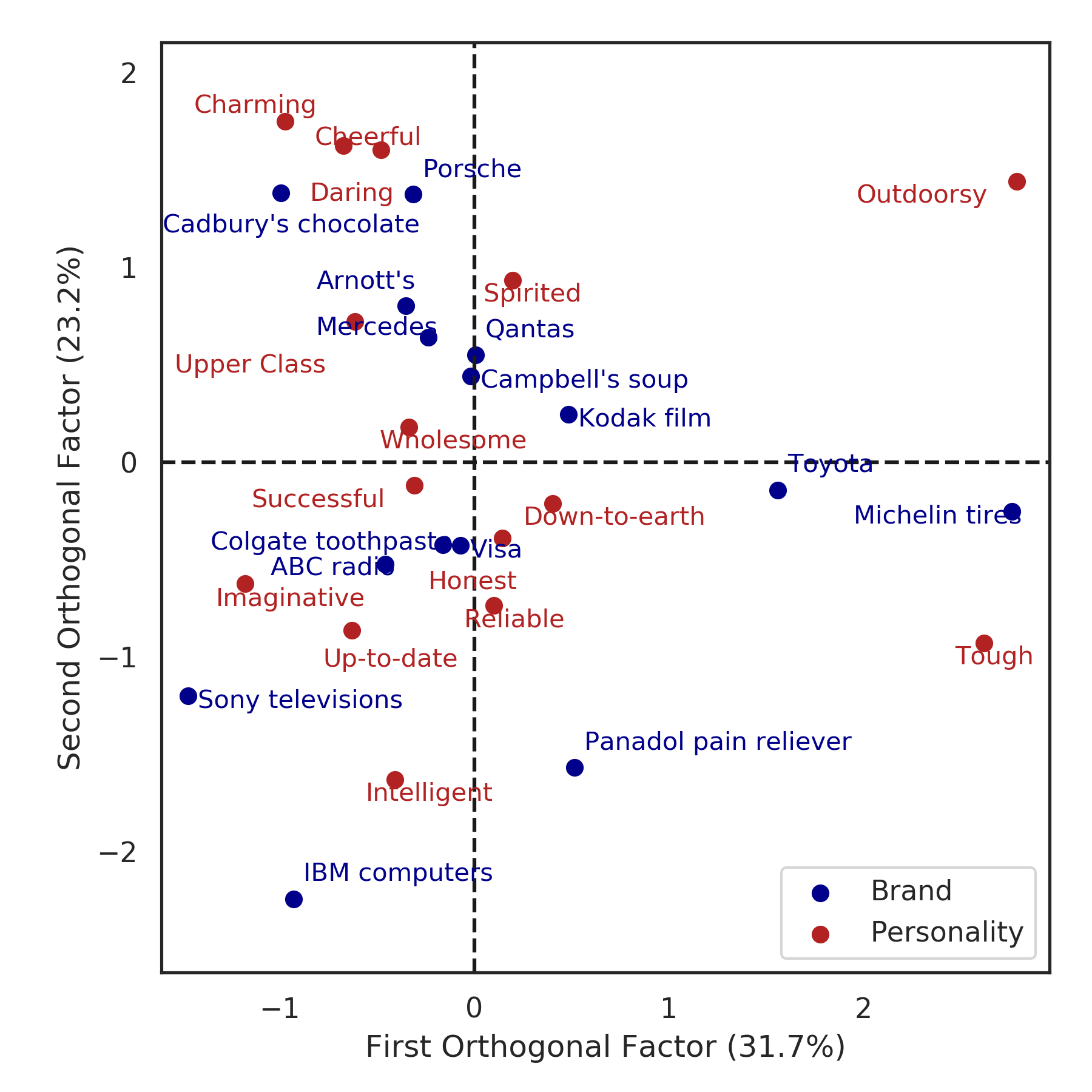}
\vspace{-.1in}
\caption{Correspondence analysis on parts of the brand-personality dataset \cite{aaker1997dimensions}.}
\label{fig:brand}
\end{figure}

\subsection{A Case Study on CA}
In order to better explain  CA, we provide a brief example next.
We use the brand-personality dataset \cite{aaker1997dimensions}, where individual subjects of a study were asked to assign different brands $X$ (e.g., IBM, Colgate) a set of personality traits $Y$ (e.g., intelligent, upper-class). 
We randomly collect the assignments for $15$ brands and $15$ personalities, and perform the correspondence analysis shown in Fig.~\ref{fig:brand}.
This visualization is much easier to interpret than, for example, a $15\times 15$ table with relative number of traits assigned per brand. In Fig.~\ref{fig:brand}, each blue dot corresponds to a two-dimensional embedding produced for a brand $X$, and each red dot correspond to an embedding produced for a personality trait $Y$. 
The position of each embedding is given by the first and second orthogonal factors of $X$ and $Y$, i.e., each blue and red points are $([\mathbf{L}]_{i, 1}, [\mathbf{L}]_{i, 2})$ and $([\mathbf{R}]_{i, 1}, [\mathbf{R}]_{i, 2})$ for $i \in [15]$ respectively. The factor score ratios for the first and second orthogonal factors are $31.7\%$ and $23.2\%$, meaning that the first two factors capture approximately half of the total correlation.
We can see that the brands ``Mercedes'' and ``Arnott's'' are upper-class; ``Porsche'' has high correspondence with upper-class, daring, spirited, and cheerful.
Moreover, since the ``successful'' and ``wholesome'' are close to the origin (i.e., point $(0, 0)$), they may not be good traits to differentiate between the brands.

After seeing a first example usage of CA, we provide next the definition of the PICs, which will enable the CA decomposition in (\ref{eq:Q}) to be performed for arbitrary random variables (under appropriate compactness assumptions), thus circumventing the need to perform a SVD of a contingency table. The results in the next section imply, for example, that the two-dimensional embedding displayed in Fig. \ref{fig:brand}---which was produced by computing the SVD \eqref{eq:Q} for the brand-personality dataset---corresponds exactly to the maximally correlated principal functions $\left\{(f_i,g_i)\right\}_{i=1,2}$ given in \eqref{eq:non-linear2}.

\subsection{Functional Spaces and the PICs}\label{sec:pics_principal functions}
For a random variable $X$ over the alphabet $\mathcal{X}$, we let $\calL_2(P_X)$ be the Hilbert Space of all functions from $\mathcal{X} \to \mathbb{R}$ with finite variance with respect to $P_X$, i.e.,
\begin{equation}
    \calL_2(\Px) \triangleq \left\{f:\calX\to \Reals \;\middle| \; \|f(X)\|_2<\infty \right\}.
\end{equation} 
For example, the output of a neuron in a feed-forward neural network is a point in $\calL_2(\Px)$.
This Hilbert space has an associated inner product given by $\langle f_1, f_2 \rangle = \mathbb{E}_X[ f_1(X)f_2(X)]$, $f_1, f_2 \in \mathcal{L}_2(P_X)$.
As customary, this inner product induces a distance\footnote{For convenience, and without loss of generality, we let the distance be defined as $d(f,g) = \langle f-g,f-g\rangle $, as opposed to $d(f,g) = \sqrt{\langle f-g,f-g\rangle }$.} between two functions $f_1, f_2 \in \calL_2(P_X)$, namely the MSE distance given by 
\begin{equation}
    d(f_1,f_2) = \mathbb{E}_X\left[ (f_1(X) - f_2(X))^2\right].
\end{equation}
For pairs of random variables $(X,Y) \sim P_{X,Y}$ taking values in $\mathcal{X} \times \mathcal{Y}$, we can similarly define the Hilbert Space $\calL_2(P_{X,Y})$. Note that $\calL_2(P_X)$ and $\calL_2(P_Y)$ are subspaces of $\calL_2(P_{X,Y})$ and, thus, one can construct the projection operator from  $\calL_2(P_X)$ to  $\calL_2(P_Y)$ given by 
\begin{eqnarray}\label{eq:compact}
    \Pi_{Y=y} [f] &\triangleq& \underset{g \in \calL_2(P_Y)}{\mathrm{argmin}} \; \mathbb{E}_{X,Y}\left[ ( f(X) - g(Y))^2\right| Y=y] = \mathbb{E}[f(X)|Y = y],
\end{eqnarray}
with adjoint operator $\Pi_{X=x}[g] = \mathbb{E}[g(Y)|X = x]$  defined for $g \in \calL_2(P_Y)$. The projection operator describes the closest function, in terms of mean-square-error, to a given function $f$ of the inputs.
Since $\calL_2(P_X)$ is a Hilbert space, there exists a basis (in fact, infinitely many) through which any function $f \in \calL_2(P_X)$ can be equivalently represented.
However, for applications such as the CA (see Section~\ref{sec:gca} to Section~\ref{sec:mv_mm} for details), it is of interest to find a basis for $\calL_2(P_X)$ which \emph{diagonalizes} the projection operator $\Pi_{Y=y}$.
This naturally leads to the following proposition.

\begin{thm}[\hspace{-.02em}{\cite{witsenhausen1975sequences}}]
\label{thm:defnPIC}
Without loss of generality, let $|\calY| \leq |\calX|$ and let $d \triangleq |\calY| -1$, or be infinity if both $|\calX|$ and $|\calY|$ are unbounded. There exists two sets of functions $\mathcal{F} = \{ f_0, f_1, \ldots,f_{d}\} \subseteq \calL_2(\Px)$ and $\mathcal{G} = \{g_0, g_1, \ldots g_{d}\}\subseteq\calL_2(\Py)$, and a set $\calS = \{1,\lambda_1,\ldots, \lambda_{d}\}$ such that:
\begin{itemize}
    \item \textbf{Orthornormality}: $f_0(X)$ and $g_0(Y)$ are constant function almost surely,  $\mathbb{E}[f_i(X) \cdot f_j(X)] = \delta_{i,j}$ and $\mathbb{E}[g_i(Y) \cdot g_j(Y)] = \delta_{i,j}$, for all $0 \leq i, j \leq d$.
    \item \textbf{Diagonalization}: $\mathbb{E}[f_i(X)|Y=y] = \sqrt{\lambda_i}g_i(y)$, and $\mathbb{E}[g_i(Y)|X=x] = \sqrt{\lambda_i} f_i(x)$ for all $1 \leq i \leq d$.
    \item \textbf{Basis}: Any function $g \in \calL_2(P_Y)$ can be represented as a linear combination $g(y) = \sum_{i = 0}^{d} \beta_i g_i(y)$. Similarly, any function $f \in \calL_2(P_X)$ can be represented as a linear combination $f(x) = f^{\perp}(x) + \sum_{i = 0}^{d} \alpha_i f_i(x)$, where $f^{\perp}$ is orthogonal to all $f_i$ for all $i = 0,1,\ldots,d$.
\end{itemize}
\end{thm}

Borrowing the terminology from CA \cite{greenacre1984theory}, the functions within the sets $\mathcal{F}$ and $\mathcal{G}$ are defined here as the \emph{principal functions} of $P_{X,Y}$, and the elements of $\calS$ as the \emph{Principal Inertia Components}\footnote{The PICs can be intuitively understood as a measure of inertia of the posterior distribution vectors on the probability simplex, and hence the name, see, e.g., \cite{du2017principal}.} (PICs) of $P_{X,Y}$. 
Observe that $0\leq \lambda_i\leq 1$ and, without loss of generality, we let $1 \geq \lambda_1 \geq \lambda_2 \geq \ldots \geq \lambda_{d}$. 
We denote $f_i$ and $g_i$ the $i^\text{th}$ principal functions, and $\lambda_i$ the $i^\text{th}$ PIC; moreover, the $0^\text{th}$ principal functions are the constant $1$ (i.e., $f_0(x) = g_0(y) = 1$ almost surely, for all $x \in \calX$ and $y \in \calY$), and the resulting $0^\text{th}$ PIC $\lambda_0$ always has value $1$.

The principal functions constitute a particularly useful set of non-linear and orthornormal embedding of $X$ that can be reliably reconstructed from $Y$ (and vice-versa). 
The decomposition in Theorem~\ref{thm:defnPIC} allows  minimum-MSE (MMSE) estimators to be easily cast in terms of the principal functions of $X$. Consider the problem of estimating an arbitrary function $g$ of the labels. Since $\mathcal{G}$ forms a basis, $g$ can be written as $\sum_{i=0}^{d} \beta_i g_i(y)$, and thus, the best estimator of $g$ from the features $X$ is given by \cite{du2017principal}
\begin{equation}
\label{eq:mmse}
\underset{f\in \calL_2(\Px)}{\mathrm{argmin}}\; \EE{\left(f(X)-g(Y)\right)^2} = \sum_{i=0}^{d} \beta_i \sqrt{\lambda_i} f_i(x). 
\end{equation}

The $i^\text{th}$ principal function $f_i$ and $g_i$, as well as the PIC $\lambda_i$, have an equivalent recursive characterization:
\begin{cor}[\hspace{-.02em}{\cite{du2017principal, witsenhausen1975sequences}}]
\label{corr:pic}
Given $X$ and $Y$, assuming $f_0(x) = g_0(y) = 1, \forall x \in \calX, y \in \calY$, and let $d = \min\{|\calX|, |\calY|\}-1$, the PIC $\lambda_i$ is defined as
\begin{eqnarray}
\lambda_i(X; Y) =\hspace{-1em} &\max\limits_{\substack{f_i \in \calL_2(\Px)\\ g_i \in \calL_2(\Py)}}& \mathbb{E}[f_i(X)g_i(Y)]^2,\; 1 \leq i \leq d\\
&\text{subject to}& \mathbb{E}[f_i(X)f_j(X)] = 0, 0 \leq j  \leq i-1, \nonumber\\
& & \mathbb{E}[g_i(Y)g_j(Y)] = 0, 0 \leq j  \leq i-1. \nonumber
\end{eqnarray}
\end{cor}

Note that the procedure in Corollary~\ref{corr:pic} is similar to that of PCA, which recursively determines the orthogonal directions that preserve variance. The PICs, in turns, determine the orthonormal functions that are maximally correlated in the Hilbert space. The principal functions $f$ and $g$ can be viewed as non-linear mappings that embed $X$ and $Y$ into a common space, in which $f(X)$ and $g(Y)$ can be viewed as embeddings.

The PICs are also connected to other information-theoretic measures.
For instance, the sum of the PICs $\sum_{i=1}^d \lambda_i$ is equal to the $\chi^2$-divergence $\chi^2(X; Y)$ between $X$ and $Y$, i.e., $\chi^2(X; Y) = \sum_{i=1}^d \lambda_i$ \cite{du2017principal}.
In this sense, the maximal correlation is a special case of the $\chi^2$ divergence when $d = 1$, and provides local approximations of $f$-divergence \cite{makur2015linear, csiszar2004information}.

\subsection{The Lancaster Decomposition}\label{sec:lancaster-decomposition}
As illustrated in \eqref{eq:mmse}, the principal functions precisely characterize the MSE-performance of estimating a function of $X$ from an observation $Y$ (and vice-versa) \cite{du2017principal}. 
In fact, the PICs and principal functions can be used to reconstruct the joint distribution entirely, as stated in the following important theorem which introduces the \emph{Lancaster decomposition}\footnote{This  decomposition has also appeared in the CA literature \cite{greenacre1984theory, buja1990remarks} under the name \emph{reconstitution formula}.}. This decomposition is key to connect the PICs with CA.

\begin{thm}[Lancaster Decomposition \cite{lancaster1958structure}]
\label{thm:reconstitution_formula}
Given the PICs $\{\lambda_i\}_{i=1}^d$ and the principal functions $\{f_i, g_i\}_{i=1}^d$ of the joint distribution $P_{X, Y}$ with regard to a common measure, we have
\begin{equation}\label{eq:reconstitution_formula}
\begin{aligned}
\frac{\Pxy(x, y)}{\Px(x)\Py(y)} &= \sum_{(f_i,g_i)\in \calF\times\calG} \sqrt{\lambda_i} f_i(x) g_i(y) = 1 + \sum_{i=1}^d \sqrt{\lambda_i} f_i(x) g_i(y).
\end{aligned}
\end{equation}
\end{thm}
This decomposition can be traced back to Lancaster's study on the decomposition of discrete joint distribution using polynomials \cite{lancaster1958structure}, which was later generalized by Hannan to the continuous case in \cite{hannan1961general}.
The Lancaster decomposition is at the core of the applications of PICs and principal functions studied here, including CA, classification boundary visualization, and multi-view and multi-modal learning, as we are about to discuss next.

\subsection{Generalizing Correspondence Analysis}\label{sec:generalizing_CA}
The Lancaster decomposition (Theorem~\ref{thm:reconstitution_formula}) is the main tool for generalizing CA to continuous variables. 
We make this connections precise in the following theorem, which demonstrates that the orthogonal factors found in CA are \emph{exactly} the principal functions.
\begin{thm}\label{thm:CA_PIC}
If $|\calX|$ and $|\calY|$ are finite, we set $[\mathbf{F}]_{i, j} = f_j(i)$, $[\mathbf{G}]_{k, j} = g_j(k)$ for $1 \leq i \leq |\calX|$, $1 \leq j \leq d$ and $1 \leq k \leq |\calY|$, and let $\mathbf{\Lambda} = \mathsf{diag}(\lambda_1, \cdots, \lambda_d)$.
Moreover, let $\mathbf{L}$, $\mathbf{R}$ and $\bSigma$ follow from \eqref{eq:svd_Q} and assume the diagonal entries of $\bSigma$ are in descending order. 
Then, the principal functions $\mathbf{F}$ and $\mathbf{G}$ are equivalent to the orthogonal factors $\mathbf{L}$ and $\mathbf{R}$ in the CA, and the factoring scores $\bSigma$ are the same as the PICs $\mathbf{\Lambda}$.
\end{thm}
\begin{proof}
See Appendix~\ref{appendix:ca_pic}.
\end{proof}

In other words, CA produces a PIC decomposition for discrete random variables $X$ and $Y$ over a finite support. 
Note, however, that the results in Section~\ref{sec:pics_principal functions} (Corollary~\ref{corr:pic} in particular) indicate that the orthogonal factors sought by CA could be computed for arbitrary distributions without the need of decomposing a contingency table \eqref{eq:contengency_table} as long as we can solve optimizations of the forms \eqref{eq:mmse} in theory and \eqref{opti4} in practice. 
This fact will enable CA to be performed at a large scale with continuous random variables by using the PIC estimator developed in Section~\ref{sec:pice}.

\section{Learning Principal Functions from Data}\label{sec:pice}
In the previous section, we demonstrated that the orthogonal factors found via CA (using SVD) are equivalent to the principal functions given by the PIC decomposition of $P_{X,Y}$ (Theorem~\ref{thm:CA_PIC}). 
Thus, we can (at least in theory)  perform CA by computing principal functions directly, without having to build a contingency table first. 
Principal functions, in turn, are well-defined for both discrete and continuous (or mixed) $X$ and $Y$, enabling CA to be extended to a broader range of data types. 
Thus, generalizing CA boils down to the problem of estimating the PICs and principal functions from $n$ data samples (realizations)  $\{x_k, y_k\}_{k=1}^n$ drawn i.i.d. from an unknown joint distribution $P_{X, Y}$.
For the rest of the paper, we use the term principal functions and PICs to indicate the orthogonal factors and factor scores, respectively.

Recall that \eqref{eq:mmse},~\eqref{eq:reconstitution_formula}, and Theorem~\ref{thm:defnPIC} suggest that the principal functions can be computed for arbitrary variables by finding maximally correlated functions in $\calL_2(P_X)$ and $\calL_2(P_Y)$. 
Finding such functions, however, requires a search over the space of all finite-variance functions of $X$ and $Y$, which is not feasible for high dimensional data. 
Thus, in order to approximate the principal functions and the PICs, we restrict our search to \emph{functions representable by neural networks}. 
We make use of the fact that the output of any neuron of a feed-forward neural network that receives $X$ as an input can be viewed\footnote{We assume that the outputs of a neural network have finite variance. This is a reasonable assumption since several gates used in practice have bounded value (e.g., sigmoid, tanh) and, at the very least, the output is limited by the number of bits used in floating point representations.} as a point in $\calL_2(P_X)$ (and equivalently for networks receiving $Y$ as input). 
Note that, in general, it is hard to analytically determine the principal functions except for special cases (see Section~\ref{sec:validation} for examples based on binary and Gaussian distributions, and \cite{buja1990remarks, makur2016polynomial} for further examples).

In this section, we design the Principal Inertia Component Estimator (PICE) using neural networks as a vessel to search over $\calL_2(P_X)$ and $\calL_2(P_Y)$.
The PICE estimates the PICs and the principal functions given samples drawn from $\Pxy$ by minimizing an appropriately defined loss function (described next) using stochastic gradient descent and backpropagation over deep neural networks. 
Moreover, we test the PICE on two synthetic datasets, and show that the PICE can reliably recover the PICs and the principal functions predicted by theory.

\subsection{Optimization}
For two random variables $(X,Y)$ (e.g., sample/label, distinct views of an image), consider the following $d$ functions of $X$ and $Y$ respectively, 
\begin{equation}
    \begin{aligned}
    \mathbf{\tilde{f}}(X) &\triangleq [\tilde{f}_1(X), \cdots, \tilde{f}_d(X)]^\intercal \in \Reals^{d\times 1},\\
    \mathbf{\tilde{g}}(Y) &\triangleq [\tilde{g}_1(Y), \cdots, \tilde{g}_d(Y)]^\intercal \in \Reals^{d\times 1}.
\end{aligned}
\end{equation}
When $X$ or $Y$ has finite support, it suffices to consider $d = \min\{|\mathcal{X}|, |\mathcal{Y}|\}-1$ (cf.\ Theorem~\ref{thm:defnPIC} in Section~\ref{sec:deep_ca}). 
Under these assumptions, the solution of the optimization problem
\begin{equation}\label{opti2}
\begin{aligned}
\min\limits_{\bA \in \Reals^{d\times d},\mathbf{\tilde{f}},\mathbf{\tilde{g}}} &\; \mathbb{E}\left[\|\bA\mathbf{\tilde{f}}(X)-\mathbf{\tilde{g}}(Y)\|^2_2\right]\\
\text{subject to}&\; \mathbb{E}\left[\bA\mathbf{\tilde{f}}(X)(\bA\mathbf{\tilde{f}}(X))^\top \right] = \mathbf{I}_d
\end{aligned}
\end{equation}
recovers the $d$ largest PICs.
To see why this is the case, let
\begin{equation}
    \bff(X)=\bA\mtf(X)=[\bff_1(X),\cdots,\bff_d(X)]^\top,
\end{equation}
and suppose that $\bff,\mtg$ and $\bA$ achieve optimality in \eqref{opti2}. Optimality under quadratic loss implies that $\tilde{g}_i(y)=\EE{f_i(X)\mid Y=y}$ for $i\in \{1,\dots,d\}$. 
Moreover, the orthogonality constraint ensures that the entries of $\bff(X)$ satisfy $\EE{f_i(X)f_j(X)}=\delta_{i,j}$, and thus form a basis for a  $d$-dimensional subspace of $\calL_2(\Px)$. 
As discussed in Section~\ref{sec:introduction}, conditional expectation on $Y$ is a (compact) operator from $\calL_2(\Px)\to \calL_2(\Py)$ (Eq.~\ref{eq:compact}) and, from orthogonality of $\bff(X)$, it follows directly from the Hilbert-Schmidt Theorem \cite[Theorem~VI.16]{reed1980functional} that the optimal value of \eqref{opti2} is $\sum_{i=1}^{d}\lambda_i$, with $\bff$ corresponding to the $d$ largest principal functions.

We can further simplify the objective function in (\ref{opti2}) using the following theorem.
The proof is based on the orthogonal Procrustes problem \cite{gower2004procrustes}, whose convergence properties have been studied in \cite{nie2017generalized}.
\begin{thm}\label{thm:opti}
The minimization in (\ref{opti2}) is equivalent to the following unconstrained optimization problem: 
\begin{equation}\label{opti4}
\begin{aligned}
\min\limits_{\mathbf{\tilde{f}}, \mathbf{\tilde{g}}} && - 2\|\bC_f^{-\frac{1}{2}}\bC_{fg}\|_d + \mathbb{E}[\|\mathbf{\tilde{g}}(Y)\|^2_2],
\end{aligned}
\end{equation}
where $\bC_f = \mathbb{E}[ \mathbf{\tilde{f}}(X)\mathbf{\tilde{f}}(X)^\top ]$, $\bC_{fg} = \mathbb{E}[ \mathbf{\tilde{f}}(X)\mathbf{\tilde{g}}(Y)^\top ]$, and $\|\bZ\|_d$ is the $d$-th Ky-Fan norm, defined as the sum of the singular values of $\bZ$ \cite{horn1990matrix}. Denoting by $\bA$ and $\bB$ the whitening matrices\footnote{We call $\bA$ and $\bB$ the whitening matrices since in (\ref{corr:pic}) it is clear that the covariance matrices of $\bff(X)$ and $\bg(Y)$ should be identity matrices.}
for $\mathbf{\tilde{f}}(\bX)$ and $\mathbf{\tilde{g}}(\bY)$, the principal functions are given by $\bff(X) = [f_0(X), \cdots, f_d(X)]^\top = \bA\mtf(X)$ and $\bg(Y) = [g_0(Y), \cdots, g_d(Y)]^\top = \bB\mtg(Y)$.
\end{thm}
\begin{proof}
See Appendix~\ref{appendix:opti}.
\end{proof}
Next, we show how the optimization \eqref{opti4} can be approximated by neural networks given samples drawn from $P_{X, Y}$.

\subsection{Implementation}\label{sec:implementation}
Observe that \eqref{opti4} is an unconstrained optimization problem over the space of all finite variance functions of $X$ and $Y$. 
As discussed previously in this section, we restrict our search to functions given by outputs of neural networks, parameterizing $\mathbf{\tilde{f}}(X)$ and $\mathbf{\tilde{g}}(Y)$ by  $\theta_F$ and $\theta_G$, respectively. Here, $\theta_F$ and $\theta_G$ denote weights of two neural networks, called the F-Net and the G-Net (Fig. \ref{fig:fg_nets}). 
The F-Net and the G-Net embed $X$ and $Y$ in $\Reals^d$, respectively. The parameters of each network can be found using gradient-based back-propagation of the objective (\ref{opti4}), as described next.

\begin{figure}[t!]
\centering
\includegraphics[width=.6\textwidth]{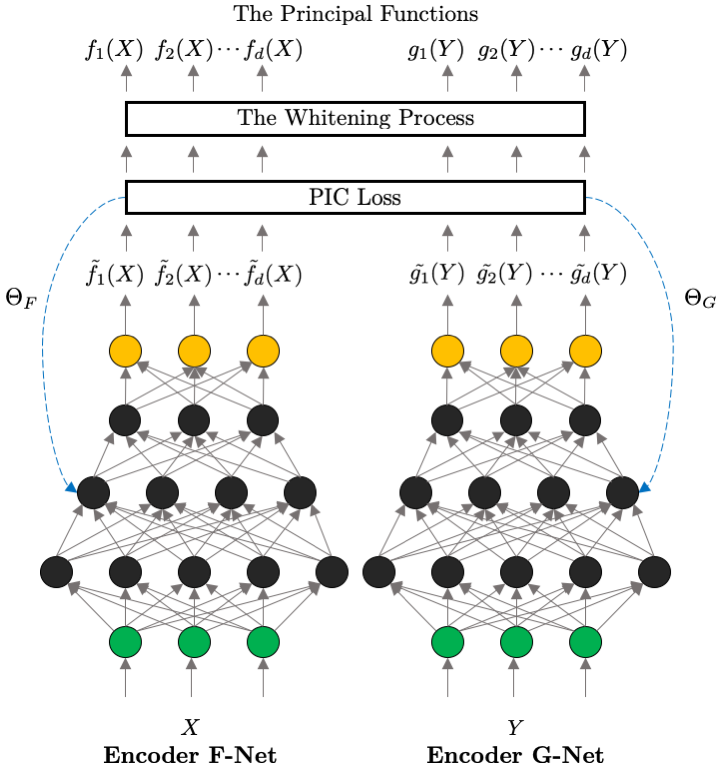}
\caption{The architecture of the PICE, consisting of two encoders F-Net and G-Net for $X$ and $Y$, respectively, to estimate the principal functions.  The PIC loss is given by (\ref{opti4}) and (\ref{eq:empirical_loss}), and can be back-propagated through F-Net and G-Net for the weights $\Theta_F$ and $\Theta_G$ simultaneously. The principal functions are then reconstructed by the whitening process in Appendix~\ref{appendix::algo}. }
\label{fig:fg_nets}
\end{figure}

Given $n$ realizations (samples) $\{x_k, y_k\}_{k=1}^n$ from $\Pxy$, we denote $\bx_n \triangleq [x_1, \cdots, x_n]$, $\by_n \triangleq [y_1, \cdots, y_n]$. 
The outputs from the FG-Nets can then be denoted as
\begin{eqnarray}
\begin{aligned}
\tilde{\bF}_n(\bx_n) &= [\mathbf{\tilde{f}}(x_1, \theta_F), \cdots, \mathbf{\tilde{f}}(x_n, \theta_F)]^\top \in \Reals^{d\times n},\\
\tilde{\bG}_n(\by_n) &= [\mathbf{\tilde{g}}(y_1, \theta_G), \cdots, \mathbf{\tilde{g}}(y_n, \theta_G)] \in \Reals^{d\times n}.
\end{aligned}
\end{eqnarray}
The empirical evaluations of the terms in (\ref{opti4}) are
\begin{subequations}\label{eq:empirical_loss}
\begin{eqnarray}
\bC_f &\approx& \frac{1}{n} \tilde{\bF}_n(\bx_n, \theta_F) \tilde{\bF}_n(\bx_n, \theta_F)^\top, \\
\bC_{fg} &\approx& \frac{1}{n} \tilde{\bF}_n(\bx_n, \theta_F) \tilde{\bG}_n(\by_n, \theta_G)^\top, \\
\mathbb{E}[\|\mathbf{\tilde{g}}(Y)\|^2_2] &\approx& \frac{1}{n} \sum_{i=1}^n\sum_{j=1}^d \mathbf{\tilde{g}}_j(y_i, \theta_G)^2.
\end{eqnarray}
\end{subequations}
When backpropagating the objective in (\ref{opti4}), calculating the singular values of $\bC_f^{-\frac{1}{2}}\bC_{fg}$ is equivalent to calculating the eigenvalues\footnote{Recall that the singular value of $\bC_f^{-\frac{1}{2}}\bC_{fg}$ is the square root of the eigenvalues of $(\bC_f^{-\frac{1}{2}}\bC_{fg})^\top (\bC_f^{-\frac{1}{2}}\bC_{fg}) = \bC_{fg}^\top \bC_f^{-1}\bC_{fg}$.} of $\bC_{fg}^\top \bC_f^{-1}\bC_{fg}$.
The latter expression can be directly cast and backpropagated using deep learning software libraries such as Tensorflow\footnote{We note that Tensorflow has built-in methods for computing the gradient of a matrix inverse. Such gradient was also computed in \cite{wang2015deep}.} \cite{abadi2016tensorflow}.
To avoid numerical instability, we not only clip the outputs of the F-Net to the interval $[-10000, 10000]$, but also impose $\ell_2$ regularization \cite[Chap. 7.1.1]{goodfellow2016deep} on $\bC_f^{-1/2}\bC_{fg}$, i.e., we compute the eigenvalues of $\bC_{fg}^\top (\bC_f^{-1}+\epsilon \bI_d)\bC_{fg}$ instead, where $\epsilon = 0.001$ is the regularization parameter to avoid unsuccessful matrix inversion.

After extracting $\tilde{\bF}_n(\bx_n)$ and $\tilde{\bG}_n(\by_n)$ from the F-Net and G-Net, respectively, we need to reconstruct the whitening matrices $\bA$ for $\mtf$ and $\bB$ for $\mtg$ to obtain the principal functions, as suggested in Theorem~\ref{thm:opti}.
Without loss of generality, we will assume that $\tilde{\bF}_n(\bx_n)$ and $\tilde{\bG}_n(\by_n)$ have zero-mean columns, which can always be achieved by subtracting the column-mean element-wise. Then $\bA$ is given by $\bA = \bU^\top \bC_f^{-1/2} $, with $\bU$ the left singular vectors of the matrix 
\begin{equation}
    \bL = \frac{1}{n} (\bC_f^{-1/2} \tilde{\bF}_n(\bx_n))(\bC_g^{-1/2} \tilde{\bG}_n(\by_n))^\top.
\end{equation}
The matrix $\bC_f^{-1/2}$ guarantees that $\tilde{\bF}_n(\bx_n)$ has orthonormal columns, while $\bU$ rotates the set of vectors to align with $\tilde{\bG}_n(\by_n)$.
By symmetry, $\bB = \bV^\top \bC_g^{-1/2}$, where $\bV$ are the right singular vectors of $\bL$. The produced matrices $\bF_n(\bx_n) = \bA \tilde{\bF}_n(\bx_n)$ and $\bG_n(\by_n) = \bB \tilde{\bG}_n(\by_n)$ satisfy
\begin{equation}
    \frac{1}{n} \bF_n(\bx_n)^\top\bF_n(\bx_n) = \frac{1}{n} \bG_n(\by_n)^\top \bG_n(\by_n) = \bI_d,
\end{equation}
and $\frac{1}{n}\bF_n(\bx_n)^\top\bG_n(\by_n) = \mathbf{\Lambda}$ is the diagonal matrix with the estimated square roots of the PICs. It should be emphasized that in the implementation and subsequent experiments, we estimate the whitening matrices $\bA$ and $\bB$ on the training set alone, and use these estimates for test sets. For clarity, we summarize this whitening process in Appendix~\ref{appendix::algo}.

Having investigated the theoretical properties of the PICs in Section~\ref{sec:deep_ca} and designed the pipeline for estimating the principal functions via deep neural networks in this section, we are now ready to explore how the generalized CA, via PICs and the principal functions, can be used in learning problems.

\subsection{Validation of PICE on Synthetic Data}\label{sec:validation}
Before going into the applications, we validate the PICE through two synthetic examples---one on discrete data and one on continuous data.
Implementation details are provided in Appendix~\ref{appendix:exp}.

\subsubsection{Discrete Case: Binary Symmetric Channels (BSCs)}\label{sec:bsc}

\begin{table}[t]
\begin{center}
\begin{tabular}{llllllllll}
\multicolumn{5}{c}{\bf BSC PICs}\\
\hline
PICE      & $0.8011$ & $0.7942$ & $0.7918$ & $0.7883$  \\
Analytic value & $0.8000$ & $0.8000$ & $0.8000$ & $0.8000$  \\
\multicolumn{5}{c}{\bf Gaussian PICs}\\
\hline
PICE      & $0.7007$ & $0.4938$ & $0.3376$ & $0.2037$ \\
Analytic value & $0.6977$ & $0.4675$ & $0.2979$ & $0.2113$ \\
\end{tabular}
\caption{PICE reliably approximates the top four PICs in the BSC and Gaussian cases.}
\label{tab:pic_synthetic}
\end{center}
\end{table}

\begin{figure}[!tb]
    \centering
    \includegraphics[width=.8\textwidth]{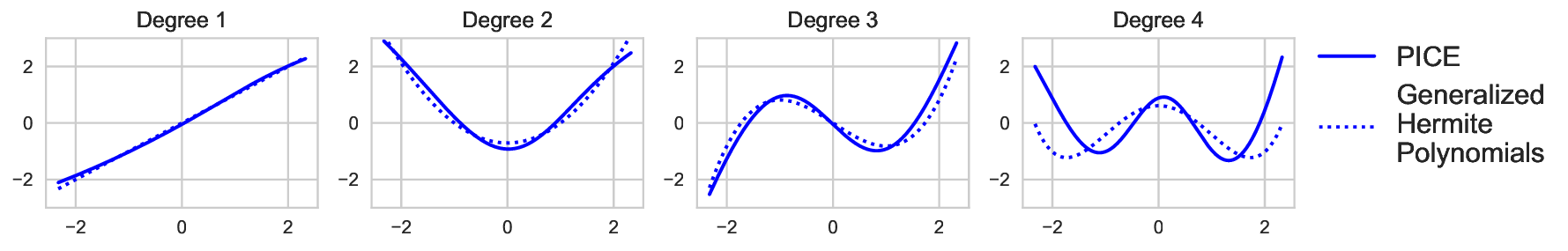}
    \caption{PICE recovers the Hermite polynomials, the principal functions in the Gaussian case.}
    \label{fig:hermite}
\end{figure}

We consider $X = (X_1,  \cdots, X_n)$ and $Z = (Z_1, \cdots, Z_n)$, where $X_i \sim Bernoulli(p)$ and $Z_i \sim Bernoulli(\delta)$, and $Y = X \oplus Z$, where $\oplus$ is the element-wise exclusive-or operator and $\delta$ is the crossover probability. 
By symmetry, it is sufficient to let $\delta \leq 1/2$. 
Note that $Y$ can be viewed as the output of $n$ uses of the discrete memoryless BSC with input $X$. 
For any additive noise binary channel, the PICs can be mathematically determined \cite[Section~2.4]{o2014analysis} or \cite[Section~3.5]{du2017principal} --- the PICs between $X$ and $Y$ have the following form: there are $\binom{n}{k}$ PICs of value $(1-2\delta)^k$, $0 \leq k \leq n$ \cite[Chap. 3.5]{du2017principal}.
We set $X$ to be a binary string of length $5$, $\delta = 0.1$, and $p = 0.1$.
The encoders F-Net and G-Net are neural networks with two hidden layers and ReLU activation, $32$ units at each hidden layer.
We train the PICE with standard gradient descent (learning rate $0.01$) across the entire training set for $2000$ epochs.
The results in Table~\ref{tab:pic_synthetic} show that the PICE reliably approximates the PICs\footnote{In this case, for $n = 5$ and $\delta = 0.1$, there are $\binom{5}{0} = 1$ PIC of value $(1-0.2)^0 = 1$, $\binom{5}{1} = 5$ PICs of value $(1-0.2)^1 = 0.8$, $\binom{5}{2} = 10$ PICs of value $(1-0.2)^2 = 0.64$, and so on.}, and we observed this consistent behaviour over multiple runs of the experiment. 

\subsubsection{Continuous Case: Gaussian Variables}\label{sec:gaussian}
When $X \sim \calN(0, \sigma_1^2\mathbf{I}_n)$, $Y|X \sim \calN(X, \sigma_2^2\mathbf{I}_n)$, the set of principal functions are the Hermite polynomials \cite{abbe2012coordinate}. More precisely, letting the $i^\text{th}$ degree Hermite polynomial be given by 
\begin{equation}
    H_i^{(r)}(x) \triangleq \frac{(-1)^i}{\sqrt{i!}} e^{\frac{x^2}{2r}} \frac{d^i}{dx^i} e^{-\frac{x^2}{2r}}, r \in (0, \infty),
\end{equation}
 then the $i^\text{th}$ principal functions $f_i$ and $g_i$ are $H_i^{(\sigma_1)}$ and $H_i^{(\sigma_1+\sigma_2)}$ respectively, and the $i^\text{th}$ PIC can then be given by their inner product. 
We pick $\sigma_1 = \sigma_2 = 1$, and generate $5000$ training samples for $X$ and $Y$ according to the Gaussian distribution and $1000$ test samples.
The PICE is composed of two hidden layers with hyperbolic tangent activation, $30$ units per hidden layer.
We train over the entire training set for $8000$ epochs using a gradient descent optimizer with learning rate $0.01$.
Table~\ref{tab:pic_synthetic} and Fig.~\ref{fig:hermite} show estimations of PICs and principal functions.

\section{Scaling Up CA using PICE}\label{sec:gca}
We illustrate next how to perform large-scale CA with continuous variables using PICE, as discussed in Section~\ref{sec:deep_ca}. 
We use PICE to produce factoring planes for the recipe \cite{kaggle_what_cooking} and wine quality \cite{asuncion2007uci} datasets described next.
All the estimations of the PICs reported in this section are averaged over $10$ random splits of the data.
The observed standard deviations were within $1\%$ of the averages.

\subsection{Kaggle What's Cooking Recipe Data}\label{sec:exp_kaggle}
\begin{table}[!tb]
{\small 
\begin{center}
\begin{tabular}{lllllllllll}
\multicolumn{1}{c}{\bf }  &\multicolumn{10}{c}{\bf Top ten principal inertia components} \\ 
\hline
PICE &  \boldsymbol{$0.9092$} & \boldsymbol{$0.8667$} & \boldsymbol{$0.8412$} & \boldsymbol{$0.7932$} & \boldsymbol{$0.7391$} & \boldsymbol{$0.6413$} & \boldsymbol{$0.6018$} & \boldsymbol{$0.4792$} & \boldsymbol{$0.4508$} & \boldsymbol{$0.2821$} \\
SVD  & $0.4504$ & $0.3894$ & $0.3149$ & $0.2943$ & $0.2413$ & $0.1958$ & $0.1547$ & $0.1191$ & $0.1146$ & $0.1035$ \\
\multicolumn{1}{c}{\bf }  &\multicolumn{10}{c}{\bf Correlations between transformed samples} \\ 
\hline
CCA  & $0.1915$ & $0.1751$ & $0.1342$ & $0.1083$ & $0.1050$ & $0.0823$ & $0.0623$ & $0.0488$ & $0.0485$ & $0.0431$ \\
KCCA & $0.6585$ & $0.1223$ & $0.0860$ & $0.0636$ & $0.0320$ & $0.0131$ & $0.0090$ & $0.0089$ & $0.0051$ & $0.0011$ \\
\end{tabular}
\end{center}
\caption{{PICE outperforms contingency table-based CA and CCA/ KCCA on Kaggle What's Cooking dataset to explore dependencies in samples.}}
\label{tab:pic_kaggle}}
\end{table}

\begin{figure*}[!tb]
\centering
\includegraphics[width=.9\textwidth]{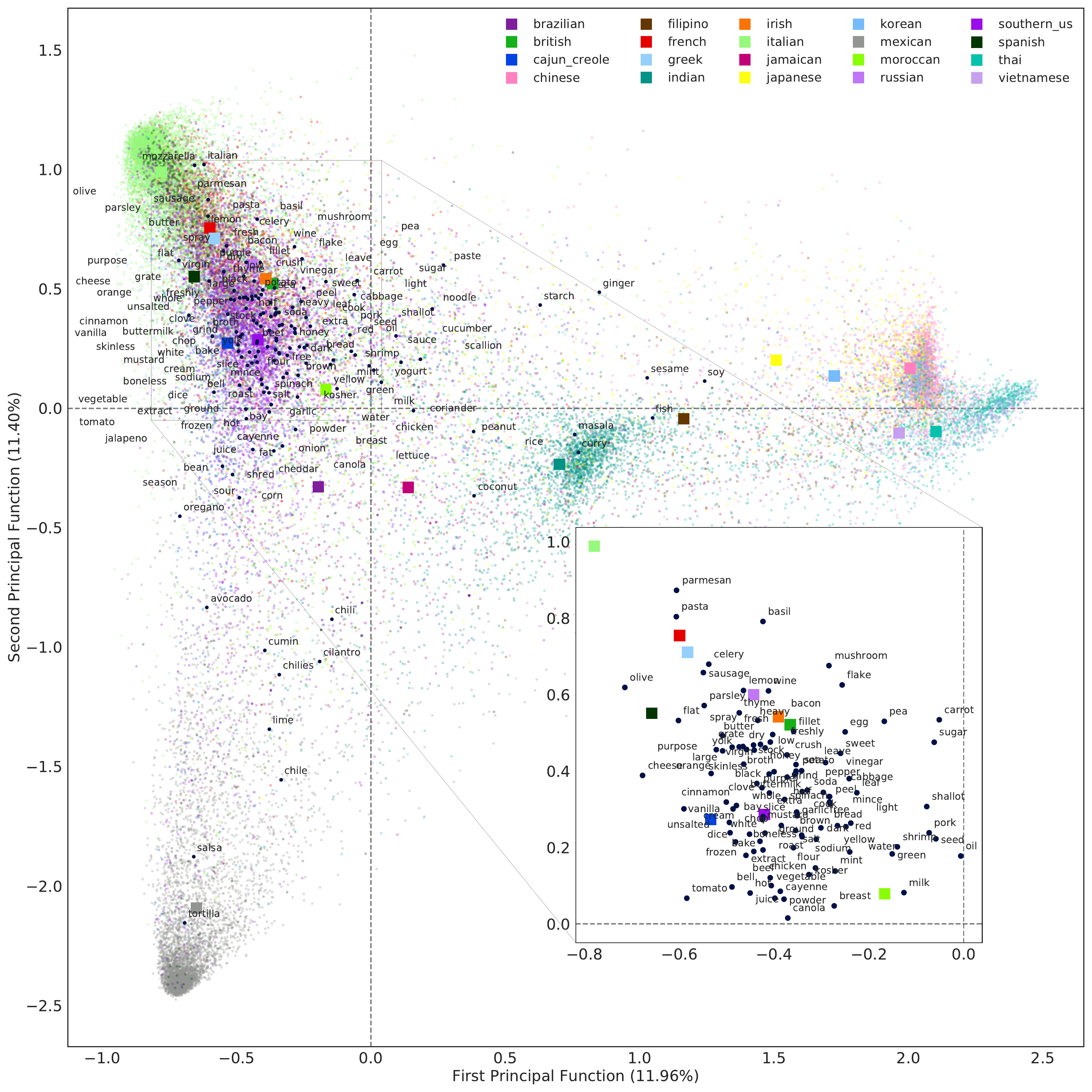}
\caption{
The first factoring plane of CA on Kaggle What's cooking dataset (Colored dots: recipe, dark blue: ingredient). The x-axis corresponds to $f_1(x)$ and $g_1(y)$, and the y-axis to $f_2(x)$ and $g_2(y)$, where each $x$ is a recipe and each $y$ are cuisine labels.
}
\label{fig:kaggle_ca}
\end{figure*}

\begin{figure*}[!tb]
\centering
\includegraphics[width=.9\textwidth]{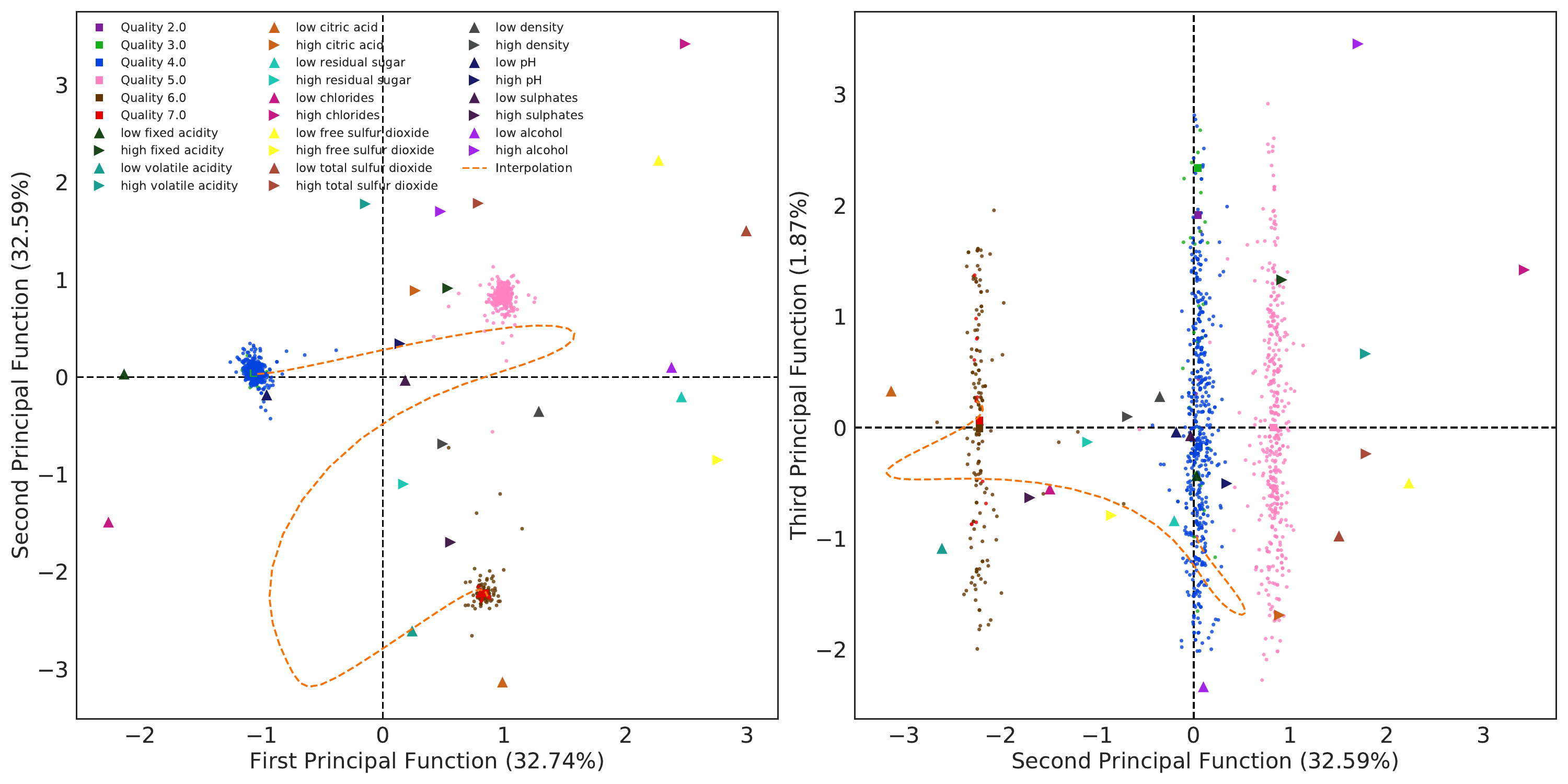}
\caption{{The first and second factoring planes on UCI wine quality dataset.}}
\label{fig:wine_ca}
\end{figure*}
The Kaggle What's cooking dataset \cite{kaggle_what_cooking} contains $39774$ recipes as $X$, composed of $6714$ ingredients (e.g., peanuts, sesame, beef, etc.), from $20$ types of cuisines as $Y$ (e.g., Japanese, Greek, Southern US, etc.). 
The recipes are given in text form, so we pre-process the data in order to combine variations of the same ingredient and keep only the $146$ most common ingredients. 
In Table~\ref{tab:pic_kaggle}, we show that PICE outperforms contingency table-based CA using SVD\footnote{We only consider combinations of ingredients observed in the dataset as possible outcomes of $X$.}, with the resulting PICs being more correlated (i.e., achieving a higher value of Eq. (\ref{opti4})) than its contingency table-based counterpart. 

In Table~\ref{tab:pic_kaggle}, we also compare the PICE with baseline techniques such as CCA and KCCA (with radial basis function kernels \cite{goodfellow2016deep}).
The resulting low-dimensional representation produced by PICE captures a higher correlation/variance than these other embeddings. We recognize that these results may vary if other kernels were selected, but note that the PICE \emph{does not require any form of kernel selection by a user prior to application}. 
In Fig.~\ref{fig:kaggle_ca} we display a traditional CA-style plot produced using PICE, showing the first factoring plane of the CA (i.e., the first and second principal functions for $X$ and $Y$). 
The intersection of two dashed lines ($x=0$ and $y=0$) indicates the space where $X$ and $Y$ have insignificant correlation.

There are three key observations which can be extracted from Fig.~\ref{fig:kaggle_ca}.
First, we observe clear clusters under the representation learned by PICE which can be easily interpreted. The cluster on the right-hand side represents East-Asian cuisines (e.g., Chinese, Korean), the one on the left-hand side represents Western cuisine (e.g., French, Italian) and in between sits Indian cuisine. 
Second, we observe that the first principal function learns to distinguish Asian cuisine (e.g., Chinese, Korean) from Western cuisine, naturally separating these contrasting culinary cultures. Interestingly, Indian cuisine sits in between Asian and Western cuisine, and  Filipino cuisine  is represented between Indian and Asian cuisine over this axis. The second principal function further indicates finer differences between Western cuisines, and singles out Mexican cuisine. 
Third, by plotting the ingredients on this plane (i.e., recipes with only one ingredient), we can determine \emph{signature} ingredients for different cuisine types. For example, despite the fact that ginger is in both Western and Asian dishes, it is closer in the factor plane to Asian cuisine, revealing that it plays a more prominent role in this cuisine. Some ingredients share much stronger correlation with the cuisine type, e.g., curry in Indian dishes and tortilla in Mexican ones.

\subsection{UCI Wine Quality Data}\label{sec:exp_wine}
The UCI wine quality dataset contains $4898$ red wines with $11$ physico-chemical attributes (e.g., pH value, acid, alcohol) and $6$ levels of qualities (from 2 to 7). 
We set $X$ to be the $11$ attributes and $Y$ be the quality, and report the results of CA in Fig.~\ref{fig:wine_ca}. 
Note that, since the attributes are continuous, performing contingency table-based CA is not well-defined for this case.
In Fig.~\ref{fig:wine_ca} (left), we can see that despite the existence of $6$ quality, the principal functions discover three sub-clusters, namely poor quality (less or equal to 4), medium quality (equal to 5), and high quality (6 and above).
Moreover, the second factoring plane (Fig.~\ref{fig:wine_ca} (right)) implies that the $11$ physico-chemical attributes of each red wine has in fact a latent dimension of $2$, since the third principal function can no longer capture effective correlations within the data.
Fig.~\ref{fig:wine_ca} also indicates how the attributes affects the quality of a wine. For example, high quality wines (quality $6$ and $7$) tends to have low citric and volatile acidity, but with high sulphates. 
Finally, we randomly sample a low quality and a high quality wine, and take the linear interpolation of their features. We represent the path that this linear interpolation draws in the factorial plane by the orange line in Fig.~\ref{fig:wine_ca}. This sheds light on how a ``bad'' wine can be transformed into a ``good'' wine, going through the ``medium'' wine category, as well as the inherent non-linearity of the principal functions.

\section{Recovering the Principal Functions from Black-Box Classification Models}\label{sec:reconstruct_pics}
Assume $\calY$ has finite cardinality $|\calY| = d$ (e.g., number of labels) and consider a classification model trained on i.i.d. samples drawn from $\Pxy$ that outputs a conditional distribution (belief) $\hat{P}_{Y|X}$.
This conditional distribution could be the output of a neural network with a softmax readout layer, a logistic regression, or a random forest. 
The belief $\hat{P}_{Y|X}$ can itself be viewed as a distribution, or a ``noisy channel'': one would hope that $\hat{P}_{Y|X}\approx P_{Y|X}$. 
This is indeed the desideratum, for example, of models trained under log-loss \cite{goodfellow2016deep}.

We demonstrate next that the principal functions and the PICs can be recovered from the ``virtual'' distribution $P_X \hat{P}_{Y|X}$ directly from fresh input samples $\{x_k\}_{k = 1}^n$ generated i.i.d. from the input distribution $P_X$.  This  leads to a Lancaster decomposition of a given classifier, and can be considered as an alternative method for computing principal functions when the dimensionality of $Y$ is low (e.g., in image recognition): first train a well-calibrated model, then approximate the  principal functions using the method below. We then use the ensuing decomposition to visualize the decision boundary and the training process of feed-forward neural networks.

\subsection{A Lancaster Decomposition for Classifiers}
A classifier (or a classification model) which ouptuts a belief $\hat{P}_{Y|X}$ is suitable for a Lancaster decomposition. Samples from the resulting ``virtual'' joint distribution $P_X \hat{P}_{Y|X}$ can be obtained by querying the classifier with samples $X$ drawn from the population distribution $P_X$, even if the classifier itself is  a black-box. In principle, this is enough to apply the PICE method and obtain principal functions and the PICs. However, the main drawback of this approach is that it would require training two neural networks (namely the F-Net and the G-Net) in addition to the already trained classifier. When the dimension of $Y$ is low, as is often customary in many classification problems, we propose next an alternative procedure for estimating the principal functions which leverages this lower-dimensional search space and circumvents the need of training additional classifiers. These principle functions can then be used to obtain a 
Lancaster Decomposition of the classifier $\hat{P}_{Y|X}$.

Consider the output (belief) produced by a classifier $\hat{P}_{Y|X}$ over $n$ samples (e.g., the output of a neural network with softmax readout later applied to $n$ samples). We denote this output by the matrix $\hat{\mathbf{P}}_{Y|X} \in \Reals^{n\times d}$. Moreover, let $\hat{\mathbf{p}}_Y \triangleq \frac{1}{n} \hat{\mathbf{P}}_{Y|X}^\top \mathbf{1}_n$ and $\hat{\mathbf{D}}_Y \in \Reals^{d\times d}$ be the diagonal matrix with diagonal given by $\hat{\mathbf{p}}_Y$.
The principal functions $g_i(Y)$ of $ \hat{P}_{Y|X} $ can then be approximated by the following optimization problem, which follows directly from the connection between the PICs and CA in Theorem \ref{thm:CA_PIC}:
\begin{eqnarray}\label{eq:recovering_from_model}
\begin{aligned}
\max\limits_{\mathbf{U} \in \Reals^{d\times d}}&\;\frac{1}{n} \text{trace}\left( \mathbf{U}^T\hat{\mathbf{D}}_Y^{-1/2}\hat{\mathbf{P}}_{Y|X}^\top \hat{\mathbf{P}}_{Y|X} \hat{\mathbf{D}}_Y^{-1/2}\mathbf{U} \right)\\
\text{s.t.}&\; \mathbf{U}^\top\mathbf{U} = \mathbf{I}_d.
\end{aligned}
\end{eqnarray}
The solution of \eqref{eq:recovering_from_model} is given by the left singular vectors $\mathbf{U}$ from the SVD of $\hat{\mathbf{D}}_Y^{-1/2}\hat{\mathbf{P}}_{Y|X}^\top \hat{\mathbf{P}}_{Y|X} \hat{\mathbf{D}}_Y^{-1/2}$, and the matrix $\mathbf{G} \in \Reals^{d\times d}$ which $(i,j)$-th entry is the $j$-th principal function $g_j$ of label $i$ is given by $\mathbf{G} = \hat{\mathbf{D}}_Y^{-1/2}\mathbf{U}$.
Indeed, if we substitute $\mathbf{G} = \hat{\mathbf{D}}_Y^{-1/2}\mathbf{U}$ back in the objective of \eqref{eq:recovering_from_model}, we are maximizing an empirical approximation of $\|\mathbb{E}[g(Y)|X]\|_2^2$ computed over $n$ samples.
Moreover, by Theorem~\ref{thm:defnPIC}, we know $f_i(x) \propto \mathbb{E}[g_i(Y)|X=x] = \sum_{y\in\calY} g_i(y) \hat{P}_{Y|X}(y|x)$ and $\EE{f_i^2(X)} = 1$; thus for all $i = 0, \cdots, d-1$ the principal functions $f_i(x)$ of $X$ is
\begin{eqnarray}\label{eq:recover_f}
f_i(x) = \frac{\sum_{y\in\calY} g_i(y) \hat{P}_{Y|X}(y|x)}{\sqrt{\frac{1}{n}\sum_{j=1}^n \left(\sum_{y\in\calY} g_i(y) \hat{P}_{Y|X}(y|x_j)\right)^2}},\;\forall x \in \calX.
\end{eqnarray}

In summary, for a given classification model $\hat{P}_{Y|X}$ and $n$ samples drawn from an input distribution $P_X$, the  pairs of principal functions for this classifier $\{f_i,g_i\}_{i=1}^d$---and, hence, the Lancaster Decomposition \eqref{eq:reconstitution_formula}---can be recovered by (i) solving \eqref{eq:recovering_from_model} for approximating $g_i$, and then recovering $f_i$ via \eqref{eq:recover_f}. We will apply this procedure in the following to subsections to visualize classification boundaries and the training of neural networks.

\subsection{Classification Boundary Visualization}\label{sec:boundary_visualization}
Producing a low dimensional representation of high-dimensional data that preserves certain relevant structures is an important problem in data science for both visualization and interpretation.
Several learning techniques have been studied to produce such dimension reduction, for example, the grand tour \cite{asimov1985grand}, the t-distributed Stochastic Neighbor Embedding (t-SNE) \cite{maaten2008visualizing}, and Uniform Manifold Approximation and Projection (UMAP) \cite{mcinnes2018umap}.
These techniques, despite being widely used, are not designed visualize the decision boundary of a given classifier. Visualizing a decision boundary, in turn, is useful for interpreting the outputs produced by a model.

Next, we perform CA of a given fixed classifier in order to visualize its decision boundary. With the recovered principal functions and PICs of a given  classification model described above, we are able to perform CA of a classifier and visualize its decision boundary by using a Lancaster decomposition.
We express the Lancaster decomposition \eqref{eq:reconstitution_formula} in the matrix form $\hat{\mathbf{P}}_{Y|X} = \mathbf{F} \mathbf{G} \mathbf{\Lambda} \hat{\mathbf{P}_Y}$, where $\mathbf{\Lambda}$ and $\hat{\mathbf{P}}_Y \in \Reals^{d\times d}$ are $\mathsf{diag}(\bm{\lambda})$ and $\mathsf{diag}(\hat{\mathbf{p}}_Y)$, and $\bm{\lambda} = [\lambda_0, \cdots, \lambda_{d-1}]^\top$.
The matrix form suggests that given $\mathbf{G}$ and $\mathbf{\Lambda}$ recovered from a classifier, the principal functions $\mathbf{F}$ corresponding to the classification likelihood $\hat{\mathbf{P}}_{Y|X}$ can be reconstructed by
\begin{eqnarray}\label{eq:find_boundary}
\mathbf{F} = \hat{\mathbf{P}}_{Y|X} (\mathbf{G}\mathbf{\Lambda} \hat{\mathbf{P}}_Y)^{-1}.
\end{eqnarray}
In particular, when $\hat{\mathbf{P}}_{Y|X}$ is selected as the set of likelihoods that gives ambiguous decision\footnote{For example, in a $3$-class classification problem, the likelihood $[0.4, 0.4, 0.2]$ gives ambiguous decisions between the first and second classes.}, the reconstructed principal functions $\mathbf{F}$ fall on the decision boundaries of classifier used to obtain $\mathbf{G}$ and $\mathbf{\Lambda}$.

We illustrate this approach on a classifier trained on the CIFAR-10 dataset  \cite{krizhevsky2009learning}. 
We pick the images of ``truck'', ``automobile'' and ``'ship'' for the sake of illustration\footnote{Since there are only two principal functions in this case the full decomposition is captured in a two-dimensional plane.}, obtain training accuracy $90.24\%$ and test accuracy $86.83\%$, and recover the $\mathbf{G}$ and $\mathbf{\Lambda}$ from the classifier. 
Since there are three classes, the three decision boundaries come from the likelihoods of the forms $[p, p, 1-2p]$, $[p, 1-2p, p]$ or $[1-2p, p, p]$, where $p \in [1/3, 1/2]$.   
We uniformly sample $100$ points $p_i$ in $[1/3, 1/2]$, build $[\hat{\mathbf{P}}_{Y|X}]_{i, \cdot} = [p_i, p_i, 1-2p_i]$ for $i \in [100]$, and recover the principal functions $f_i$ corresponding to the decision boundary from the likelihoods $[p, p, 1-2p]$ by \eqref{eq:find_boundary}.
Similarly, we can recover the principal functions corresponding to other decision boundaries.
In Fig.~\ref{fig:boundary}, we can see that the first principal function distinguishes between automobile/truck and ship, which is an easier relation to distinguish compared to automobile and truck (second principal function). 
Moreover, we also dislpay the image samples that are close to the decision boundary in order to identify which images ``confuse'' this classifier.
In this example, the images of a automobile parking at the beach or a truck with large blue sky (pointed by the red arrows in Fig.~\ref{fig:boundary}) can be easily confused with a ship, indicating that a potential feature captured by the classifier to identify ship is having a large portion of blue  in an image.
Note that unlike t-SNE \cite{maaten2008visualizing} or UMAP \cite{mcinnes2018umap}, which preserve the relative geometric distance between samples, the proposed method in \eqref{eq:find_boundary} preserves the difference of the likelihood, i.e., if $\hat{P}_{Y|X=x_1} \approx \hat{P}_{Y|X=x_2}$, then $f_i(x_1) \approx f_i(x_2)$ for all $i$.
Moreover, the principal functions corresponding to the decision boundary are linear since \eqref{eq:find_boundary} is a linear transformation of the (non-linear) principal functions.

\begin{figure}[!tb]
\centering
\includegraphics[width=.6\textwidth]{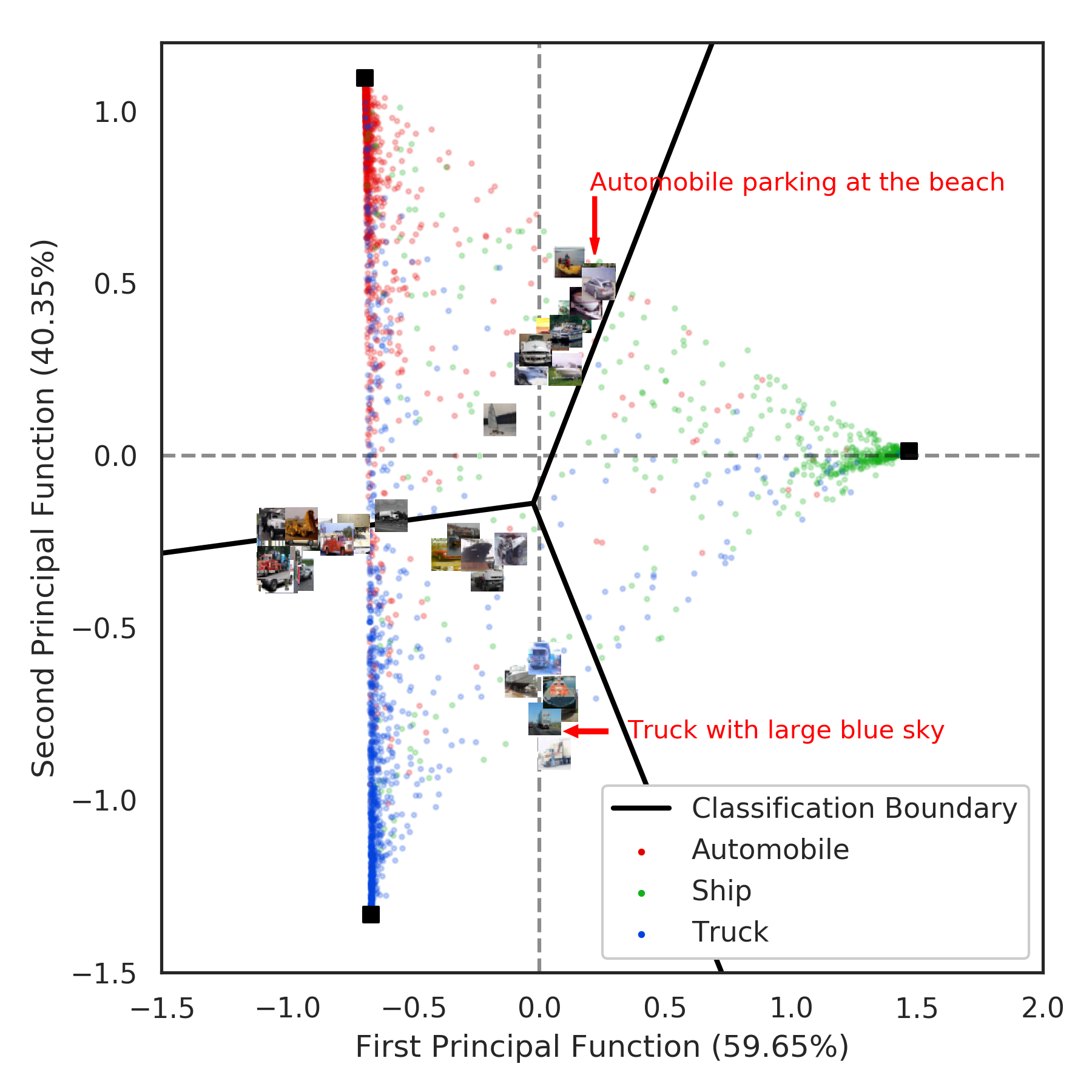}
\caption{{The first factoring planes of samples from three classes in CIFAR-10, the decision boundary, and the confusing images.}}
\label{fig:boundary}
\end{figure}

\subsection{Visualization the Training Process}\label{sec:exp_training_process}
Understanding the training process of a neural network is one of the main challenges in machine learning research \cite{gish1990probabilistic, shwartz2017opening, guidotti2018survey}. 
For example, does the classifier learn to distinguish all the classes one at a time or simultaneously? What is the class that is most difficult to classify during training or the last one learned?
We illustrate next that by performing CA of a classifier during its training process these questions can be partially answered. 

We use as examples the MNIST \cite{lecun1998gradient} and CIFAR-10 \cite{krizhevsky2009learning} datasets.
The MNIST dataset has $60000$ images for training and $10000$ for test, where each image has $28 \times 28$ pixels handwritten digit (from $0$ to $9$) in gray-scale.
We train a classifier using the \textit{AdamOptimizer} \cite{kingma2014adam} based on a simple convolutional neural network, and for each mini-batch, we store the output of the softmax readout layer, reconstruct the principal functions, and concatenate the corresponding CA plots of test data into a video available at \url{https://www.youtube.com/watch?v=RIfxdAwhtrI}.
The final accuracy for training set is $95.30\%$ and $95.01\%$ for test set.
As we can see, before the $50$th batch, the neural network is unable to perform accurate classification; however, after that the classifier suddenly distinguishes the samples in each classes.
This indicates that for the MNIST dataset, the classifier undergoes a phase transition during learning.

In contrast, the CIFAR-10 dataset, which is a more challenging dataset than MNIST, contains $32 \times 32$ colored images in $10$ categories (e.g., dog, cat, truck, car), with $50$K/$10$K images for training/testing.
In \url{https://www.youtube.com/watch?v=8mGiXamxcEU}, we can see that the classifier first learns to distinguish between ship/airplane, automobile/truck and the rest. 
Subsequently, it learns to differentiate ship from airplane and truck from automobile gradually in the remaining of the training steps.
The accuracy for training set is $68.35\%$ and $61.68\%$ for test set.

\section{Multi-View and Multi-Modal Learning}\label{sec:mv_mm}
In the previous subsection, we considered $X$ as samples and $Y$ as classification labels. 
However, the PICs are not restricted to this use case scenario; in fact, $X$ and $Y$ can be drawn from the same data type.
A relevant example of this setting is given by multi-view learning.

\subsection{Latent Dimensionality in Multi-View Learning}\label{sec:multiview}

Multi-view learning \cite{hu2018sharable, zhang2018generalized, wang2015deep} is a popular approach to discover features used in classification tasks.
Taking images as an example, consider pictures of the same object taken from two (or more) distinct perspectives, e.g., different angles, light conditions. These two views correspond to $X$ and $Y$.
A central task of multi-view learning is to identify a common latent space for the multiple views. 
This is classically done by constructing a low-dimensional embedding of $X$ and $Y$, and then applying a clustering algorithm on this low-dimensional embeddings, e.g., \cite{wang2015deep}. 
These common embeddings of the views are then analyzed in order to identify features that are relevant in a classification task at a later stage \cite{benton2017deep, hardoon2004canonical}.
 
A key question in this unsupervised setting is \emph{what is the dimension of the common latent space shared by both images?}
In effect, this is analogous to determining a basis for the latent space shared by the views.  
The choice of the number of clusters is often done via heuristics \cite{wang2016deep, benton2017deep, hardoon2004canonical}.
We explain next how the PICs can be used to determine the features that are common to both images, as well as the number of such features, thus giving a principled method to answer the question above.

Multi-view (two-view) learning can be formalized as follows. 
We say $X$ and $Y$ are multiple views of the same object $W$ if $X$ and $Y$ are independent given $W$, i.e., $X - W - Y$, where $W$ takes value in a finite set $\mathcal{W} = [M]$.
Under this assumption, and given a dataset $\mathcal{D} = \{(x_k, y_k)\}_{k = 1}^n$ where each pair $(x_k, y_k)$ is drawn i.i.d. according to the distribution $\Pxy$, the goal of multi-view learning is to recover:
\begin{enumerate}
    \item[(i)] the dimension of the latent variable, i.e., $M$;
    \item[(ii)] two sets of features, $\{\phi_m(X)\}_{m=1}^M$ for the first view $X$ and $\{\varphi_m(Y)\}_{m=1}^M$ for the second view $Y$ such that they are maximally correlated.
\end{enumerate}
Multi-view learning lies naturally within the PIC analysis.
If we compute the first $K \geq M$ principal functions $\{f_i(X), g_i(Y)\}_{i=1}^K$ for the two views using, for example, the PICE, then $\{\phi_m(X)\}_{m=1}^M \subseteq \{f_i(X)\}_{i=1}^K$ and $\{\varphi_m(Y)\}_{m=1}^M \subseteq \{g_i(Y)\}_{i=1}^K$.
Moreover, since there is only a $M$-dimensional latent space shared by $X$ and $Y$, the first $M$ PICs take higher values than the rest of the PICs. Thus, a sudden drop in the value of the PICs is an indication of the size of the latent space.

\begin{figure}[!tb]
    \centering
    \includegraphics[width=0.6\textwidth]{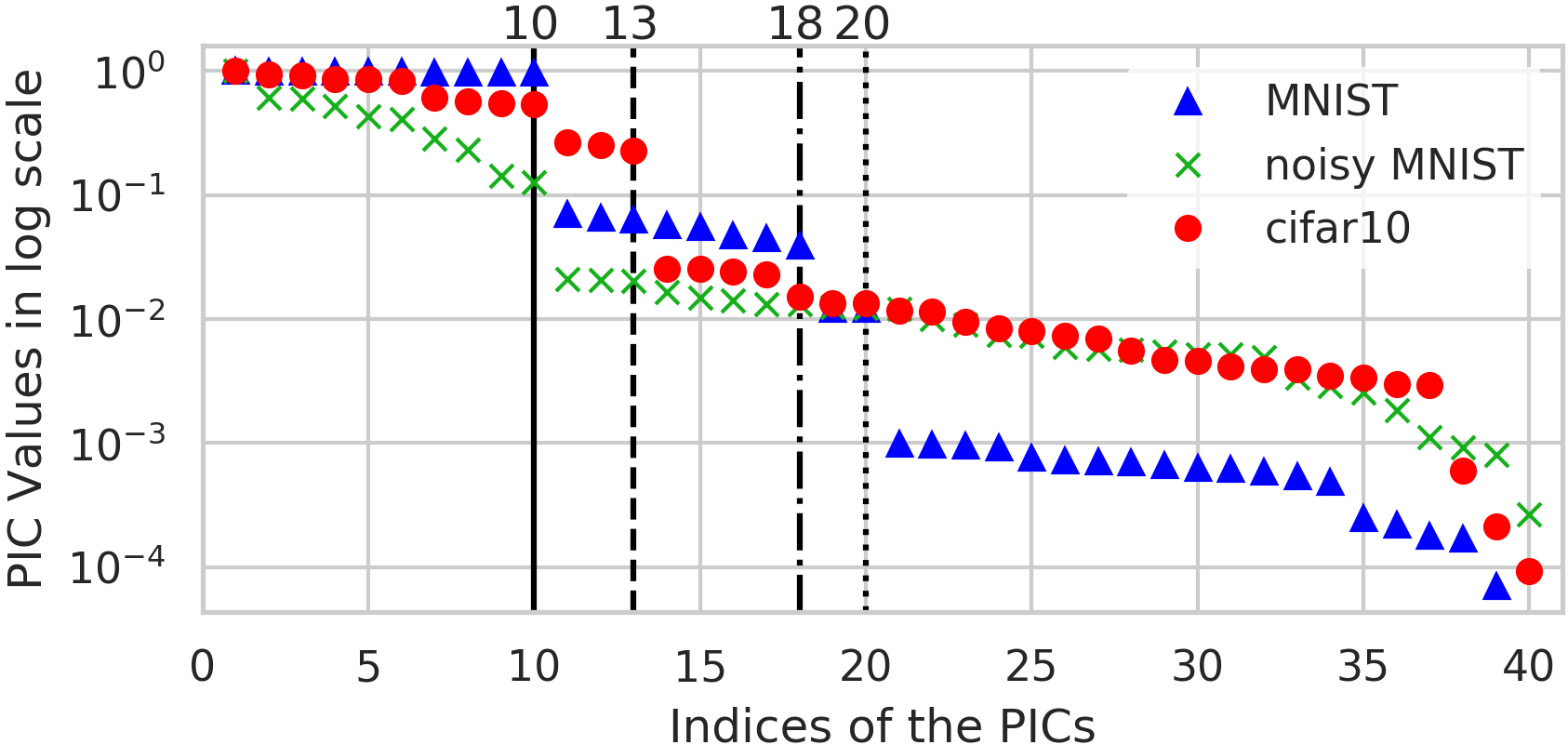}
    \caption{The first $40$ PICs of the two views of images.}
    \label{fig:mv}
\end{figure}

We show next how the PICE can be used to empirically quantify the dimensionality of the common latent space shared by different views in multi-view representation learning \cite{chaudhuri2009multi,arora2013multi} on three datasets: MNIST \cite{lecun1998gradient}, noisy MNIST \cite{wang2015deep}, and CIFAR-10 \cite{krizhevsky2009learning}.
We extract $40$ principal functions for each dataset.

\begin{table}[!tb]
\begin{center}
\begin{tabular}{llllllllll}
\multicolumn{8}{c}{\bf Top Eight PICs} \\ 
\hline
$0.317$ & $0.201$ & $0.095$ & $0.061$ & $0.047$ & $0.045$ & $0.031$ & $0.017$\\
\end{tabular}
\caption{Top eight PICs between images and captions from the PICE for the Flickr-$30$k dataset.}
\label{tab:pic_mm}
\end{center}
\end{table}

\begin{table*}[t!]
\centering
\begin{tabular}{| m{2.5cm} | m{11.5cm} | m{2cm} |}
\hline
\hfil Flickr Images & \hfil Original Captions & \hfil Tags \\
\hline
    \begin{center}\includegraphics[width=.6\linewidth, height=20mm]{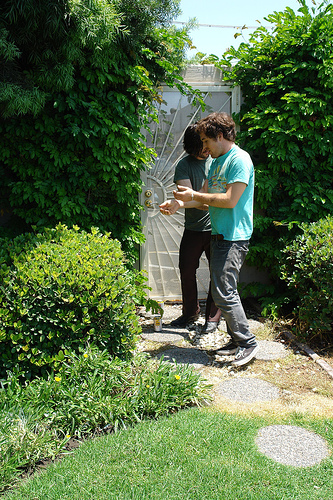}\end{center}
    & 
    \begin{itemize}
        \item Two young guys with shaggy hair look at their hands while hanging out in the yard.
        \item White males are outside near many bushes.
        \item Two men in green shirts are standing in a yard. 
        \item A man in a blue shirt standing in a garden.
        \item Two friends enjoy time spent together.
      \end{itemize}
    &
    \begin{itemize}
        \item standing
        \item people
        \item talking
        \item blue shirt
        \item green
    \end{itemize}
    \\
\hline
    \begin{center}\includegraphics[width=.9\linewidth, height=15mm]{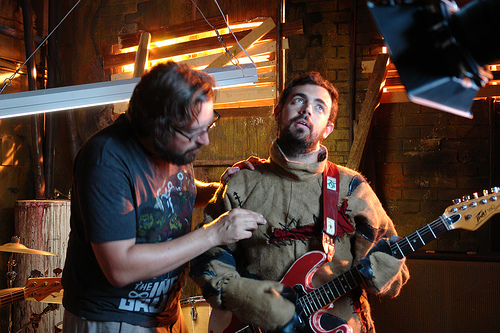}\end{center}
    & 
    \begin{itemize}
        \item Two people in the photo are playing the guitar and the other is poking at him.
        \item A man in green holds a guitar while the other man observes his shirt.
        \item A man is fixing the guitar players costume. 
        \item A guy stitching up another man 's coat .
        \item The two boys playing guitar.
      \end{itemize}
    &
    \begin{itemize}
        \item players
        \item guitar
        \item stand
        \item young
    \end{itemize}
    \\
\hline
    \begin{center}\includegraphics[width=.9\linewidth, height=15mm]{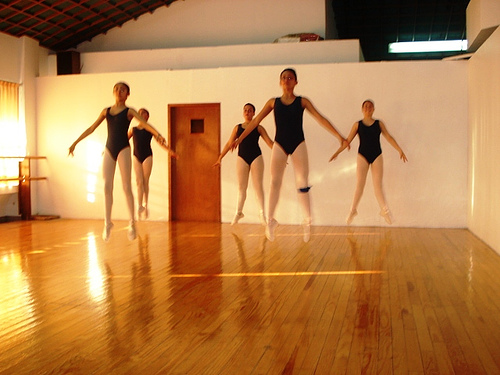}\end{center}
    & 
    \begin{itemize}
        \item Five ballet dancers caught mid jump in a dancing studio with sunlight coming through a window.
        \item Ballet dancers in a studio practice jumping with wonderful form.
        \item Five girls are leaping simultaneously in a dance practice room. 
        \item Five girls dancing and bending feet in ballet class .
        \item A ballet class of five girls jumping in sequence.
      \end{itemize}
    &
    \begin{itemize}
        \item playing
        \item girls
        \item lady
        \item orange
        \item yellow
    \end{itemize}
    \\
\hline
\end{tabular}
\caption{Image-to-tag task using PICE on the Flickr-$30$k dataset.}
\label{fig:image2tag}
\end{table*}

\begin{table*}[t!]
\centering
\begin{tabular}{| m{1.0cm} | m{2.5cm} | m{2.5cm} | m{2.5cm} | m{2.5cm} |}
\hline
\hfil Tag & \hfil Image Sample $1$ & \hfil Image Sample $2$ & \hfil Image Sample $3$ & \hfil Image Sample $4$ \\
\hline
    \hfil Snow
    & 
    \begin{center}\includegraphics[width=.9\linewidth, height=15mm]{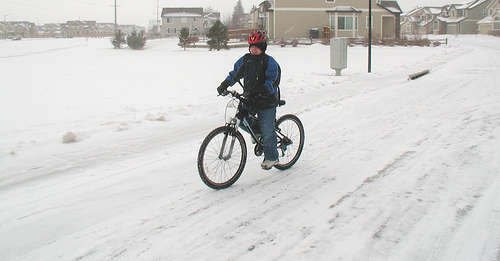}\end{center}
    &
    \begin{center}\includegraphics[width=.9\linewidth, height=15mm]{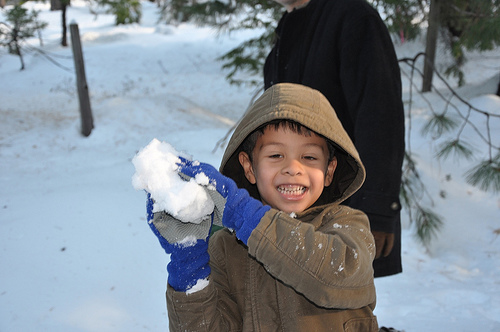}\end{center}
    &
    \begin{center}\includegraphics[width=.9\linewidth, height=15mm]{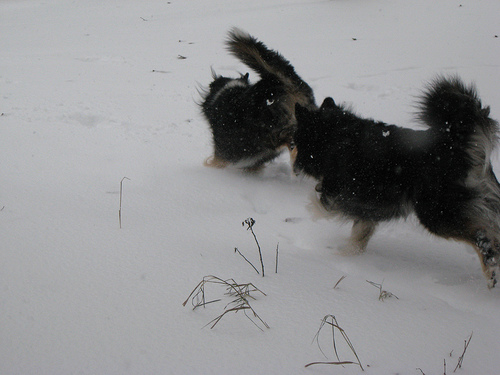}\end{center}
    &
    \begin{center}\includegraphics[width=.9\linewidth, height=15mm]{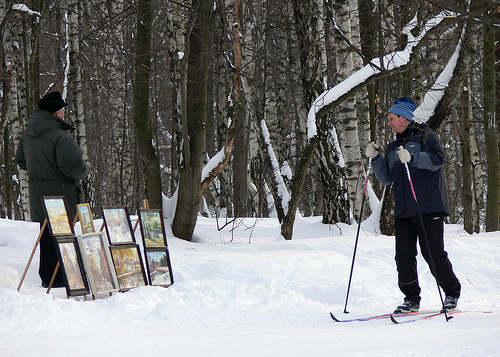}\end{center}
    \\
\hline
    \hfil Water
    & 
    \begin{center}\includegraphics[width=.9\linewidth, height=15mm]{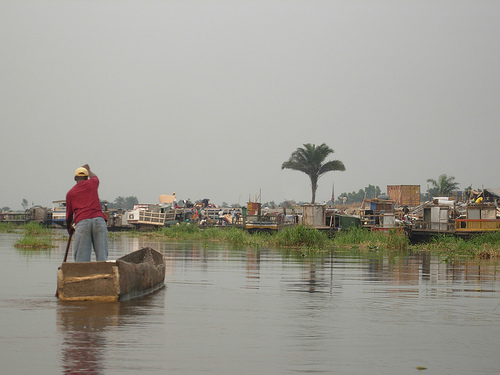}\end{center}
    &
    \begin{center}\includegraphics[width=.9\linewidth, height=15mm]{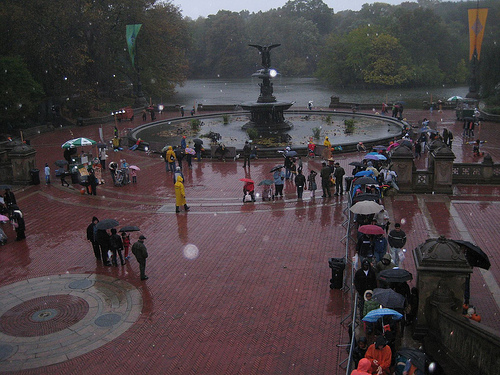}\end{center}
    &
    \begin{center}\includegraphics[width=.9\linewidth, height=15mm]{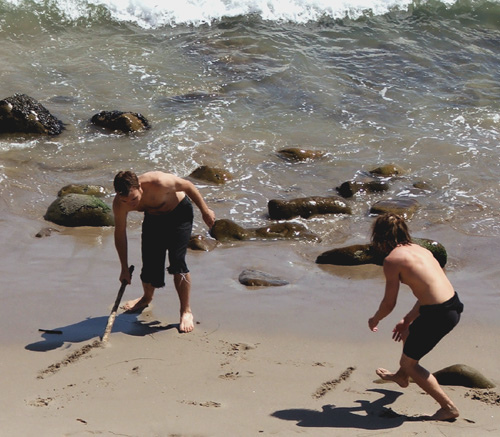}\end{center}
    &
    \begin{center}\includegraphics[width=.9\linewidth, height=15mm]{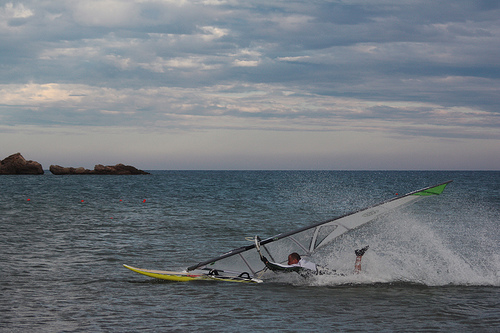}\end{center}
    \\
\hline
\end{tabular}
\caption{Tag-to-Image task using PICE on the Flickr-$30$k dataset.}
\label{fig:tag2image}
\end{table*}

\textbf{MNIST.}\;
We randomly select two images of the same digit as the two views $X$ and $Y$.
The PICE is trained for $50$k epochs on the training set, with a batch size of $2048$ for estimating the correlation/cross correlation matrix. 
It is reasonable to expect $10$ significant PICs in this case (due to $10$ underlying digits). 
As seen in Fig.~\ref{fig:mv}, there is a significant drop in the PIC values precisely at $10$, indicating that the dataset indeed has $10$ significant (nonlinear) components.

\textbf{Noisy MNIST.}\;
The dataset is the noisy version of the MNIST dataset, in the sense that for view 1 ($X$), each image is rotated at angles uniformly sampled from $[-\pi/4, \pi/4]$, and for view 2 ($Y$), each image is randomly rotated and random noise uniformly sampled from $[0, 1]$ is added.
The PICE is trained similarly as the experiment of MNIST. 
Again, in Fig.~\ref{fig:mv}, we observe a significant drop in the PICs values at $10$; however, due to the random notation and noise, the PICs are individually smaller than those from the experiment of MNIST.

\textbf{CIFAR-10.}\;
We sample view $1$ ($X$) and view $2$ ($Y$) by randomly selecting pairs from a randomly chosen category in each batch. 
Once again in Fig.~\ref{fig:mv}, we observe a drop of the PICs at around $13$, indicating $13$ directions that maximally preserve correlation after training. 
Note that since the CIFAR-10 dataset is more complicated than the (noisy) MNIST dataset, these directions do not necessarily correspond to the underlying labels (e.g., may distinguish between images with dark/light background as opposed to dog/cat).

\subsection{Multi-Modal Learning}
Multi-modal learning (see e.g., \cite{li2018survey, ngiam2011multimodal, srivastava2012multimodal}) is similar to the multi-view learning problem introduced above, with the key difference that $X$ and $Y$ have different \emph{modalities}, i.e., they come from different domains. 
For example, $X$ could be an image and $Y$ could be a piece of a caption (e.g., a sentence) describing the content of the image.
The main challenge in multi-modal learning is to find a set of common low-dimensional mappings such that the data from different modalities are maximally correlated, which is usually achieved by existing methods like CCA/KCCA \cite{gong2014multi}. 

The principal functions are a useful tool to perform this task.
We focus on the specific scenario where $X$ are images and $Y$ are captions.
When the principal functions $f_i$ and $g_i$ map the images and captions into a common latent space, given an image $x_k$, we are able to compute which group of keywords $y_k$ (extracted from the captions) is closest to $x_k$ in this low-dimensional space.
Given a radius $\tau > 0$, we define the set 
\begin{eqnarray}
\calY(x_k) = \{ y_k|\; \|\bff(x_k)- \mathbf{g}(y_k)\|_2 \leq \tau, \forall k = 1, \cdots, n \}.
\end{eqnarray}
The set $\calY(x_k)$ contains all the keywords that are maximally correlated to the image $x_k$, and can be used as the tag or annotation for the image.
In other words, the principal functions can automatically ``tag'' the images (i.e., the image-to-tag task \cite{jeon2003automatic}); moreover, this methodology is also useful when we want to annotate unlabeled images \cite{ando2005framework}.
Similarly, given a keyword $y_k$, we are able to construct a set $\calX(y_k)$ that contains the images closest to this keyword.
This tag-to-image task was recently applied to search images in a dataset given keywords \cite{gong2014multi}.

We implement the image-to-tag and tag-to-image discussed in Section~\ref{sec:multiview} using the PICE on the Flickr-$30$k image caption dataset \cite{plummer2015flickr30k}.
This dataset consists of $32$k high-resolution images obtained from the Flickr website, each of which has been annotated with five descriptive captions.
We resize all the images ($X$) to $64\times 64$ pixels for the purpose of training neural networks, and pre-process the captions to keep only meaningful terms (i.e., pieces of words) and use a bag-of-words model \cite{manning2010introduction} to tokenize each caption according to term frequency, ending up with $500$ words ($Y$).
We extract $50$ principal functions using the PICE, and report the top eights PICs in Table~\ref{tab:pic_mm}.
Note that the PICs are not as large as previous examples since all the captions are actually composed of $20586$ words; however, we truncate the words to $500$, losing some correlations between images and captions as a compromise for implementation.

For illustration, we randomly select $3$ test samples and find the $5$ words closest to these images in the latent space formed by the $50$ principal functions to illustrate the image-to-tag task in Table~\ref{fig:image2tag}.
We observe that the principal functions capture information that is not provided in the original captions but are related with the image. 
For example, in Table~\ref{fig:image2tag} row~$1$, the tags include ``talking'', which is an information not provided in the captions; yet another example is in Table~\ref{fig:image2tag} row~$3$, where the tags ``orange'' and ``yellow'' are not stated in the captions but are obviously describing the tone of the image.
Similarly, in Table~\ref{fig:tag2image}, we pick two tags, ``snow'' and ``water'', to perform the tag-to-image task with $\tau = 0.1$ and randomly select four images.
These examples show that the images in the latent space formed by the $50$ principal functions are indeed close to the tags.

\section{Conclusion}\label{sec:conclusion}
We generalize correspondence nalysis to an unprecedented scale using a theoretical framework for  analyzing and producing low-dimensional orthonormal mappings of pairs of dependent random variables that are maximally correlated called the PICs. 
This framework also extends classical statistical metrics and techniques such as maximal correlation and the ACE algorithm. 
Moreover, the decomposition of conditional distributions \eqref{eq:reconstitution_formula} in CA can be used to visualize classification boundary, as well as understand the evolution of models during training, and underlies recent multi-view/multi-modal learning methods such as DCCA.
We hope that this information-theoretic PIC framework for generalizing CA can be a useful tool in machine learning and data science and extended to different learning tasks.

\bibliographystyle{IEEEtran}
\bibliography{reference.bib}

\appendix
\section{Proof of Theorem~\ref{thm:CA_PIC}}\label{appendix:ca_pic}
If we write (\ref{eq:reconstitution_formula}) into matrix form, we have
\begin{eqnarray}\label{eq:pic_to_ca}
\mathbf{F} \mathbf{\Lambda}  \mathbf{G}^\intercal &=& \mathbf{D}_{X}^{-1}\mathbf{P}_{X,Y}\mathbf{D}_{Y}^{-1}-\mathbf{1}_{|\calX|}\mathbf{1}_{|\calY|}^\intercal\\
&=& \mathbf{D}_{X}^{-1}(\mathbf{P}_{X,Y}-\mathbf{p}_X\mathbf{p}_Y^\intercal)\mathbf{D}_{Y}^{-1}\\
&=& \mathbf{D}_{X}^{-1/2}\mathbf{Q}\mathbf{D}_{Y}^{-1/2}\\
&=& \mathbf{D}_{X}^{-1/2}\bU \bSigma \bV^\intercal\mathbf{D}_{Y}^{-1/2}\\
&=& \mathbf{L}\bSigma \mathbf{R}^\intercal,
\end{eqnarray}
where $[\mathbf{F}]_{i, j} = f_j(i)$, $[\mathbf{G}]_{i, j} = g_j(i)$ and $\mathbf{\Lambda} = \mathsf{diag}(\lambda_0, \cdots, \lambda_d)$.
Eq.~(\ref{eq:pic_to_ca}) shows that in discrete case, the principal functions $\mathbf{F}$ and $\mathbf{G}$ are equivalent to the orthogonal factors $\mathbf{L}$ and $\mathbf{R}$ in the CA, and the factoring scores $\bSigma$ are the same as the PICs $\mathbf{\Lambda}$. 
\section{Proof of Theorem~\ref{thm:opti}}\label{appendix:opti}
Since $\mathbb{E}[\|\bA\mathbf{\tilde{f}}(X)-\mathbf{\tilde{g}}(Y)\|^2_2] = \text{tr} \left( \bA\mathbb{E}[ \mathbf{\tilde{f}}(X)\mathbf{\tilde{f}}(X)^\intercal ]\bA^\intercal \right) - 2\text{tr} \left( \bA\mathbb{E}[ \mathbf{\tilde{f}}(X)\mathbf{\tilde{g}}(Y)^\intercal] \right) + \left( \mathbb{E}[\|\mathbf{\tilde{g}}(Y)\|^2_2] \right)$, we have 
\begin{equation}
    \mathbb{E}[\|\bA\mathbf{\tilde{f}}(X)-\mathbf{\tilde{g}}(Y)\|^2_2] = d - 2\text{tr} \left( \bA \bC_{fg} \right) + \mathbb{E}[\|\mathbf{\tilde{g}}(Y)\|^2_2],
\end{equation}
where the last equation comes from the fact that $\text{tr} \left( \bA\mathbb{E}[ \mathbf{\tilde{f}}(X)\mathbf{\tilde{f}}(X)^\intercal ]\bA^\intercal \right) = \text{tr} \left( \mathbf{I}_d \right) = d$. Since $\bC_f$ is positive-definite, $C_f^{-\frac{1}{2}}$ exists, and so does $\bA = \Tilde{\bA}\bC_f^{-\frac{1}{2}}$, and (\ref{opti2}) can be alternatively expressed as
\begin{equation}\label{opti3}
\begin{aligned}
\min\limits_{\Tilde{\bA}\in \Reals^{d \times d}}\ -2\text{tr}(\Tilde{\bA}\bB) + \mathbb{E}[\|\mathbf{\tilde{g}}(Y)\|^2_2],\
\text{subject to}\ \Tilde{\bA}\Tilde{\bA}^\intercal = \mathbf{I}_d,
\end{aligned}
\end{equation}
where $\bB = \bC_f^{-\frac{1}{2}}\bC_{fg}$.
The term $\text{tr}(\Tilde{\bA}\bB)$ can be upper bounded by the von Neumann's trace inequality \cite{mirsky1975trace},
\begin{equation}
    \text{tr}(\Tilde{\bA}\bB) \leq \sum_{i=1}^d \sigma_{\Tilde{\bA}, i}\sigma_{\bB, i},
\end{equation}
where $\sigma_{\Tilde{\bA}, i}$'s and $\sigma_{\bB, i}$'s are the singular values for $\tilde{\bA}$ and $\bB$ respectively. Moreover, the upper bounded can be achieved by solving the orthogonal Procrustes problem \cite{gower2004procrustes}, and the optimizer is $\Tilde{\bA}^* = \bV\bU^\intercal$, where $\bV$ and $\bU$ are given by the SVD of $\bB = \bU\mathbf{\Sigma}_\bB \bV^\intercal$.
Therefore, 
\begin{equation}
    \text{tr}(\Tilde{\bA}^*\bB) = \text{tr}(\bV\bU^\intercal \bU\mathbf{\Sigma}_\bB \bV^\intercal) = \sum_{i=1}^d \sigma_{\bB, i}
\end{equation}
which is the $d$-th Ky-Fan norm of $\bB$. The desired result then follows by simple substitution.
\section{The Whitening Process in Section~\ref{sec:implementation}}\label{appendix::algo}
{
\begin{algorithm}[H] 
\caption{Recovering $\bF_n(\bx_n)$ and $\bG_n(\by_n)$ from $\tilde{\bF}_n(\bx_n)$ and $\tilde{\bG}_n(\by_n)$, the output of the PICE.}\label{algo:whitening}
\begin{algorithmic}[1] 
\Input $\tilde{\bF}_n(\bx_n)$ and $\tilde{\bG}_n(\by_n)$
\Output Principal functions $\bF_n(\bx_n)$ and $\bG_n(\by_n)$
\State $\tilde{\bF}_n(\bx_n) \gets \tilde{\bF}_n(\bx_n) - \EE{\tilde{\bF}_n(\bx_n)}$,
\Statex $\tilde{\bG}_n(\by_n) \gets \tilde{\bG}_n(\by_n) - \EE{\tilde{\bG}_n(\by_n)}$ \Comment{(Remove mean)}
\State $\bU_f, S_f, \bV_f \gets$ SVD of $\frac{1}{n}\tilde{\bF}_n(\bx_n) \tilde{\bF}_n(\bx_n)^\intercal$,
\Statex $\bU_g, S_g, \bV_g\gets$ SVD of $\frac{1}{n}\tilde{\bG}_n(\by_n) \tilde{\bG}_n(\by_n)^\intercal$
\State $\bC_f^{-1/2} \gets \bU_f S_f^{-1/2} \bV_f^\intercal$,
\Statex $\bC_g^{-1/2} \gets \bU_g S_g^{-1/2} \bV_g^\intercal$ \Comment{(Find inverse)}
\State $\bL = \frac{1}{n} (\bC_f^{-1/2}\tilde{\bF}_n(\bx_n)) (\bC_g^{-1/2}\tilde{\bG}_n(\bx_n))^\intercal$ 
\State $\bU, S, \bV \gets$ SVD of $\bL$ \Comment{(Find singular vectors)}
\State $\bA = \bU^\intercal\bC_f^{-1/2} $, $\bB = \bV^\intercal\bC_g^{-1/2}$
\State \Return $\bA\tilde{\bF}_n(\bx_n)$, $\bB\tilde{\bG}_n(\by_n)$
\end{algorithmic}
\end{algorithm}
}
\section{Experimental Details and Setup}\label{appendix:exp}
Source code for reproducing our experimental results is given at \url{https://github.com/HsiangHsu/CorrelatedEmbeddings}, and Table~\ref{tab:parameters} summarizes all the hyper-parameters used in the experiments.
We note that a discussion on the dependence of the results on model size is included in \cite[Section~A.7]{hsu2019correspondence}.

\begin{table*}[!tbh]
\begin{center}
\small
 \begin{tabular}{||c c c c c c||} 
 \hline
 Experiments & \makecell{Neural Network\\ Architecture} & \makecell{Activation\\ Function} & \makecell{Learning\\ Rate} & \makecell{Number of\\ Training Epochs} & \makecell{Batch\\ Size} \\ [0.5ex] 
 \hline\hline
 \makecell{BSC\\(Section~\ref{sec:bsc})} & \makecell{F-Net: 30-30-25-12\\G-Net: 30-30-25-12} & $\tanh$ & $1\times 10^{-2}$ & $10$k & $1.5$k \\ 
 \hline
 \makecell{Gaussian\\(Section~\ref{sec:gaussian})} & \makecell{F-Net: 30-30-4\\G-Net: 30-30-4} & $\tanh$ & $5\times 10^{-3}$ & $8$k & $5$k \\
 \hline
 \makecell{Kaggle\\(Section~\ref{sec:exp_kaggle})} & \makecell{F-Net: 30-30-30-10\\G-Net: 30-30-30-10} & $\tanh$ & $5\times 10^{-3}$ & $20$k & $39774$ \\
 \hline
 \makecell{UCI Wine\\(Section~\ref{sec:exp_wine})} & \makecell{F-Net: 500-100-30-6\\G-Net: 10-5-3-6} & $\tanh$ & $1\times 10^{-3}$ & $1$k & $1599$ \\
 \hline
  \makecell{Decision Boundary\\CIFAR-10\\(Section~\ref{sec:boundary_visualization})} & CNN: VGG-16 & ReLU & $1\times 10^{-3}$ & $3$ & $128$ \\
 \hline
 \makecell{Training Process\\ MNIST\\(Section~\ref{sec:exp_training_process})} & \makecell{CNN\\ Convolution:\\ $[5, 5, 1, 32], 1$\\ Convolution:\\ $[5, 5, 32, 64], 1$\\ Fully Connected:\\ $[7*7*64, 1024]$\\Fully Connected:\\ $[1024, 10]$} & ReLU & $5\times 10^{-3}$ & $1$ & $125$ \\
 \hline
 \makecell{Training Process\\CIFAR-10\\(Section~\ref{sec:exp_training_process})} & CNN: VGG-16 & ReLU & $5\times 10^{-4}$ & $6$ & $500$ \\
 \hline
 \makecell{Multi-View\\MNIST\\(Section~\ref{sec:mv_mm})} & \makecell{F-Net: VGG-16\\G-Net: VGG-16} & ReLU & $1\times 10^{-2}$ & $50$k & $2048$ \\
 \hline
 \makecell{Multi-View\\Noisy MNIST\\(Section~\ref{sec:mv_mm})} & \makecell{F-Net: VGG-16\\G-Net: VGG-16} & ReLU & $1\times 10^{-2}$ & $50$k & $2048$ \\
 \hline
 \makecell{Multi-View\\CIFAR-10\\(Section~\ref{sec:mv_mm})} & \makecell{F-Net: VGG-16\\G-Net: VGG-16} & ReLU & $1\times 10^{-2}$ & $30$k & $256$ \\
 \hline
 \makecell{Multi-Model\\Flickr\\(Section~\ref{sec:mv_mm})}  & \makecell{F-Net: VGG-16\\G-Net: 100-64-32-50} & ReLU & $5\times 10^{-5}$ & $30$k & $200$ \\ [1ex] 
 \hline
\end{tabular}
\end{center}
\caption{Neural network structures used in experiments in Section~\ref{sec:validation},~\ref{sec:gca},~\ref{sec:reconstruct_pics}, and~\ref{sec:mv_mm}. 30-30-4 means the neural network consists of three layers and 30 neurons in the first layer and so on.}
\label{tab:parameters}
\end{table*}

\end{document}